\documentclass[lettersize,journal]{IEEEtran}
\usepackage{amsmath,amsfonts}
\usepackage{algorithmic}
\usepackage{algorithm}
\usepackage{array}
\usepackage{textcomp}
\usepackage{stfloats}
\usepackage{url}
\usepackage{verbatim}
\usepackage{graphicx}
\usepackage{cite}
\usepackage{microtype}
\usepackage{graphicx}
\usepackage{subfigure}
\usepackage{booktabs} 
\usepackage{color}
\usepackage{wrapfig}
\usepackage{hyperref}
\usepackage{bbm}

\usepackage{amsmath}
\usepackage{amssymb}
\usepackage{mathtools}
\usepackage{amsthm}
\usepackage{multicol}
\usepackage{multirow}
\usepackage[capitalize,noabbrev]{cleveref}

\theoremstyle{plain}
\newtheorem{theorem}{Theorem}[section]

\theoremstyle{definition}
\newtheorem{definition}[theorem]{Definition}

\theoremstyle{remark}

\usepackage{lettrine}

\begin{document}

\title{CDC: A Simple Framework for Complex Data Clustering\thanks{This work was supported by the National Natural
Science Foundation of China Grant 62276053. Corresponding author: Erlin Pan\\Z. Kang, X. Xie, B. Li, E. Pan are with the School of Computer Science and Engineering, University of Electronic Science and Technology of China, Chengdu, China (e-mail: zkang@uestc.edu.cn; x624361380@outlook.com; {bingheng86, wujisixsix6}@gmail.com).}
}
\author{Zhao Kang, ~\IEEEmembership{Member,~IEEE,} Xuanting Xie, Bingheng Li and Erlin Pan}

\markboth{Journal of \LaTeX\ Class Files,~Vol.~14, No.~8, August~2021}%
{Shell \MakeLowercase{\textit{et al.}}: A Sample Article Using IEEEtran.cls for IEEE Journals}


\maketitle

\begin{abstract}
In today's digital era driven by data, the amount and
complexity of the collected data, such as multi-view, non-Euclidean, and multi-relational,  are growing exponentially or even faster. Clustering, which unsupervisedly extracts valid knowledge from data, is extremely useful in practice. However, existing methods are independently developed to handle one particular challenge at the expense of the others. In this work, we propose a simple but effective framework for complex data clustering (CDC) that can efficiently process different types of data with linear complexity. We first utilize graph filtering to fuse geometric structure and attribute information. We then reduce complexity with high-quality anchors that are adaptively learned via a novel similarity-preserving regularizer. We illustrate the cluster-ability of our proposed method theoretically and experimentally. In particular, we deploy CDC to graph data of size 111M. 

\end{abstract}


\section{Introduction}
\lettrine{C}lustering is a fundamental technique for unsupervised learning that groups data points into different clusters without labels. It is driven by diverse applications in scientific research and industrial development, which induce complex data types \cite{lin2021multi,lv2021pseudo}, such as multi-view, non-Euclidean, and multi-relational. 
In many real-world applications, data are often gathered from multiple sources or different extractors and therefore exhibit different features, dubbed as multi-view data \cite{MVL}. Despite the fact that each view may be noisy and incomplete, important factors, such as geometry and semantics, tend to be shared across all views. Different views also provide complementary information, so it is paramount for multi-view clustering (MVC) methods to integrate diverse features. For example, \cite{MCGC2} learns a consensus graph with a rank constraint on the corresponding Laplacian matrix from multiple views for clustering. \cite{DMCC} employs intra-view collaborative learning to harvest complementary and consistent information among different views.

Along with the development of sophisticated data collection and storage techniques, the size of data increases explosively. To handle large-scale data efficiently, several MVC methods with linear complexity have been proposed. \cite{LMVSC} learns an anchor graph for each view and concatenates them for multi-view subspace clustering. \cite{SFMC}, \cite{EOMSC-CA}, and \cite{SMVSC} construct bipartite graphs or learn representations based on anchors. \cite{BMVC} effectively integrates the collaborative information from multiple views via learning discrete representations and binary cluster structures jointly. Despite these advances, they often produce unstable performance on different datasets because of the randomness in anchor selection.

Recently, non-Euclidean graph data have become pervasive since they contain not only node attributes but also topological structure information, which characterizes relations among data points \cite{kang2019robust,pan2023beyond}. Social network users, for example, have their own profiles and social relationships reflected in the topological graph. Traditional clustering methods exploit either attributes or graph structures and cannot achieve the best performance \cite{MCGC}. Graph Neural Network (GNN) is a powerful tool to simultaneously explore node attributes and structures \cite{GCN}. Based on this, several graph clustering methods have been designed \cite{AGC, DAEGC,shen2024beyond}.  In some applications, the graph could exhibit multi-view attributes or be multi-relational. To cluster multi-view graphs, \cite{MAGCN} learns a representation for each view and forces them to be close. To handle multi-relational graphs, \cite{O2MAC} finds the most informative graph to recover multiple graphs.

 Despite the remarkable success of GNN-based methods in graph clustering, there is still one crucial question, i.e., scalability, which   
 prevents their deployment to web-scale graph data. For example, ogbn-papers100M contains more than 100M nodes, which cannot be processed by most graph clustering methods. Although \cite{S3GC, liu2023dink} have made advances in scalable graph clustering by applying a light-weight encoder and contrastive learning, their performance highly depends on graph augmentation. Therefore, scalability for graph clustering is still under-explored, and more dedicated efforts are pressingly needed.
 
We can see that some specialized methods have been developed to address one of the above problems but there lacks a unified model for complex data clustering that generalizes well while still being scalable. To fill this gap, we propose a simple yet effective framework for Complex Data Clustering (CDC). We first use graph filtering to fuse raw features and topology information, which produces cluster-ability representations and provides a flexible way to handle different types of data. Rather than constructing a complete graph from all data points, the CDC learns anchor graphs, resulting in linear computation complexity. In particular, we generate anchors adaptively with a similarity-preserving regularizer to alleviate the randomness from anchor selection. Thus, the framework reduces complexity using high-quality anchors that are adaptively learned through a novel  similarity-preserving regularizer. By theoretically and experimentally illustrating the clustering capacity of CDC, we demonstrate its effectiveness in handling diverse datasets. To summarize, we make the following contributions:

\begin{itemize}
    \item We propose a simple clustering framework for complex data, e.g., single-view and multi-view, graph and non-graph, small-scale and large-scale data. Our method has linear complexity in time and space.
    \item (Section \ref{methodology}) 
    We are the first to propose a similarity-preserving regularizer to automatically learn high-quality anchors from data.
    \item (Section \ref{the}) We present the first theoretical analysis about the effect of graph filtering and the property of the learned anchor graph. Filtered representations preserve both topology and attribute similarity, and the learned graph has a grouping effect. 
    \item (Section \ref{exps}) CDC achieves impressive performance in 14 complex datasets. Most notably, it scales beyond the graph with more than 111M nodes.
\end{itemize}

\section{Related Work} \label{related}
    \subsection{Multi-view Clustering} \label{MVC}
    MVC methods generally focus on enhancing performance by utilizing global consensus and complementary information among multiple views. \cite{CoMVC} mines view-shared information by a sample-level contrastive module to align angles between representations.  \cite{xu2020modal} incorporates modal regression  to be robust to noise. \cite{FMR} uses the Hilbert-Schmidt Independence Criterion (HSIC) to explore the underlying cluster structure shared by multiple views. \cite{MMVC} generates an automatic partitioning with data from multiple views via a multi-objective clustering framework. \cite{COMIC} achieves cross-view consensus by projecting data points into a space where geometric and cluster assignments are consistent.\\
    Different from shallow methods, deep MVC methods learn good representations via designed neural networks. \cite{CMGEC} applies an attention encoder and multi-view mutual information maximization to capture the complementary information, consistency information, and internal relations of each view. Recently, some methods have combined the contrastive learning mechanism to obtain clustering-favorable representations. For example, \cite{DCP} performs contrastive learning at the instance and category levels to improve the consistency of the cross-view. 
    However, these methods are not scalable to large-scale data. To reduce complexity, \cite{LMVSC, MSGL} construct bipartite graphs between the cluster centroids of $K$-means and raw data points, where the anchors are chosen randomly and fixed for subsequent learning. \cite{FastMICE} leverages features, anchors, and neighbors together to construct bipartite graphs. \cite{FIMVC} captures the view-specific and consistent
    information by constructing a consensus graph from view-independent anchors. Although they all have linear complexity, their performance could be sub-optimal since the pre-defined anchors are not updated according to the downstream task.  
    In contrast, we generate high-quality anchors adaptively, which is efficient and stable on complex data.
    
    \subsection{Graph Clustering} \label{GC}
    Graph clustering methods group nodes based on their node attributes and topological structure. Some representation learning methods, such as Node2vec \cite{Node2vec} and GAE \cite{GAE}, can be used to learn embeddings for traditional clustering techniques. However, the obtained embeddings might not be suitable for clustering as they are not specifically designed to learning representations of cluster-ability. MVGRL \cite{MVGRL}, BGRL \cite{BGRL}, and GRACE \cite{GRACE} obtain a classification-favorable representation through contrastive graph learning, but are not applicable to large-scale graph due to computational cost associated with data augmentation. Although MCGC \cite{MCGC} is augmentation-free by considering k-nearest neighbors as positive pairs, the algorithm has quadratic complexity. Other deep graph clustering methods, such as DFCN \cite{DFCN}, \textcolor{black}{WEC-FCA \cite{bian2024weighted}},  and DCRN \cite{DCRN}, achieve promising performance by training MLP and GNNs jointly in small-scale / medium-scale graph. \textcolor{black}{DCMAC \cite{yang2024discrete}} and MvAGC \cite{MvAGC} have low complexity, but are not efficient due to their anchor sampling strategy. Hence, these graph clustering methods cannot effectively and efficiently handle large-scale graph. Though S$^3$GC \cite{S3GC} obtains promising results in the extra large-scale graph through the random walk-based sampler and the light-weight encoder, there is a high computation cost for training. Our method can handle graph clustering in linear time with promising performance. 
     Table \ref{add1} summarizes the main characteristics of some representative methods and shows how they fill the gap.

\begin{table*}[htbp]
  \centering
  \small
  \caption{The characteristics of some representative methods. “NN-based” indicates that the method is based on deep neural networks. “Comp-info” indicates that the method considers complementary information.}
\begin{tabular}{|c|c|c|c|c|c|c|c|c|}
\hline Methods & Non-graph & Multi-relation & Single-view & Multi-attribute & Unsup & Self-sup & NN-based & Comp-info\\
\hline METIS & &  &$\checkmark$ & & $\checkmark$ & & & \\
\hline Node2vec & & &$\checkmark$ & & $\checkmark$ & & & \\
\hline DGI & & & $\checkmark$ & & & $\checkmark$ & $\checkmark$ &  \\
\hline DMoN & & & $\checkmark$ & & & $\checkmark$ & $\checkmark$ & \\
\hline GRACE & & & $\checkmark$ & & & $\checkmark$ & $\checkmark$ &  \\
\hline BGRL & & & $\checkmark$ & & & $\checkmark$ & $\checkmark$ &  \\
\hline MVGRL & & & $\checkmark$ &  & & $\checkmark$ & $\checkmark$ &  \\
\hline S$^3$GC & & & $\checkmark$ & & & $\checkmark$ & $\checkmark$ &  \\
\hline CCGC & & & $\checkmark$ & & & $\checkmark$ & $\checkmark$ &  \\
\hline DAEGC & & & $\checkmark$ & &$\checkmark$ &  & $\checkmark$ & \\
\hline O2MAC & & $\checkmark$ & & &$\checkmark$ &  & $\checkmark$ &$\checkmark$ \\
\hline HDMI & & $\checkmark$ & & & & $\checkmark$ & $\checkmark$ & $\checkmark$   \\
\hline CMGEC & & $\checkmark$ & & & & $\checkmark$ & $\checkmark$ & $\checkmark$ \\
\hline COMPLETER &$\checkmark$ & & &  $\checkmark$ & & $\checkmark$ & $\checkmark$ & $\checkmark$  \\
\hline MvAGC & & $\checkmark$ & & $\checkmark$ &$\checkmark$ &  &  & $\checkmark$  \\
\hline MCGC & & $\checkmark$ & & $\checkmark$ & & $\checkmark$ & & $\checkmark$ \\
\hline MAGCN & & & & $\checkmark$ &$\checkmark$  & & $\checkmark$ & $\checkmark$ \\
\hline DMG & & $\checkmark$ & & & &$\checkmark$ &$\checkmark$ &$\checkmark$  \\
\hline BTGF & & $\checkmark$ & & & & $\checkmark$ & $\checkmark$ & $\checkmark$  \\
\hline BMVC & $\checkmark$ & & & $\checkmark$ &$\checkmark$ & &  & $\checkmark$  \\
\hline LMVSC & $\checkmark$ & & &$\checkmark$  &$\checkmark$ & &  & $\checkmark$  \\
\hline MSGL &$\checkmark$ & & & $\checkmark$ &$\checkmark$ & &  & $\checkmark$  \\
\hline FPMVS &$\checkmark$ &  & & $\checkmark$ &$\checkmark$ & &  & $\checkmark$  \\
\hline EOMSC-CA &$\checkmark$ & & &$\checkmark$  &$\checkmark$ & &  & $\checkmark$  \\
\hline FastMICE &$\checkmark$ & & &$\checkmark$  &$\checkmark$ & &  & $\checkmark$  \\
\hline CDC (Ours) &$\checkmark$ & $\checkmark$ &$\checkmark$ &$\checkmark$ &$\checkmark$ &  &  & $\checkmark$  \\
\hline
\end{tabular}
\label{add1}
\end{table*}

\section{Methodology} \label{methodology}
\paragraph{Notation}
Define the generic data as $\mathcal{G}=\lbrace \mathcal{V},E_1,...,E_{V_1},X^1,...,X^{V_2}\rbrace$, where $\mathcal{V}$ represents the set of $N$ nodes, $e_{ij}\in E_{v}$ denotes the relationship between node $ i $ and node $ j $ in the $v$-th view. $V_1\geq 0$ and $V_2> 0$ are the number of relational graphs and attributes, and the data is non-graph when initial $V_1= 0$. $X^v=\lbrace x_1^v,...,x_N^v\rbrace^{\top}\in \mathbb{R}^{N\times d_v}$ is the feature matrix, $d_v$ is the dimension of features. Adjacency matrices $ \lbrace \widetilde{A}^v \rbrace ^{V_1}_{v=1}\in \mathbb{R}^{N\times N} $  characterize the initial graph structure. For non-graph data, we construct adjacency matrices in each view via the 5-nearest neighbor method. There are $V$ views after graph filtering for each dataset, where $V=V_1\times V_2$ for graph data and $V=V_1=V_2$ for non-graph data. $ \lbrace D^v \rbrace ^{V_1}_{v =1} $ represent the degree matrices in various views. The normalized adjacency matrix is $  A^v  = (D^v)^{-\frac{1}{2}}(\widetilde{A}^v + I)(D^v)^{-\frac{1}{2}} $ and the corresponding graph Laplacian is $ L^v = I - A^v $.

\subsection{Graph Filtering}
\label{proofs}
	Filtered features are more clustering-favorable \cite{GFG}, and we apply graph filtering to remove undesirable high-frequency noise while preserving the graph's geometric features. Similarly to \cite{MCGC}, smoothed $H$ is obtained by solving the following optimization problem:
	\begin{equation} \label{smooth}
		\min _{H}\|H-X\|_{F}^{2}+\frac{1}{2} \operatorname{Tr}\left({\mathrm{H}}^\top \mathrm{LH}\right).
	\end{equation}
	 We keep the first-order Taylor series of $H$ from Eq. (\ref{smooth}) and apply $k$-order filtering, which yields:
	\begin{equation}
		H = (I-\frac{1}{2} L)^k X,
	\end{equation} 
	where $k$ is a non-negative integer and it controls the depth of feature aggregation and smoothness of representation. In addition to learning smooth features, graph filtering is also used to unify different types of data into our framework.  
 
\subsection{Anchor Graph Learning}
We use the idea of data self-expression to capture the relations among data points, i.e., each sample can be represented as a linear combination of other data points. The combination coefficient matrix can be regarded as a reconstructed graph \cite{MCGC}. To reduce the computation complexity, $m$ representative samples $B\in \mathbb{R}^{m\times d}$, called \textbf{anchors}, are selected to construct anchor graph $Z \in \mathbb{R}^{m\times N}$ \cite{LMVSC}. However, the performance of this approach is unstable since it introduces anchors in a probabilistic way. Moreover, once the anchors are chosen, they won't be updated, which could lead to sub-optimal performance. 
To eliminate uncertainty in anchor selection, we propose learning anchors from data, i.e., anchors $B$ are generated adaptively. To guarantee the quality of anchors, we enforce that the similarity between $B$ and $H$ is preserved, i.e.,  $BH^{\top}=Z$. Then we formalize the graph learning problem as:
\begin{equation}
    \min _{Z, B}\left\|H^{\top}-B^{\top} Z\right\|_F^2+\beta\left\|Z\right\|_F^2 \text {, s.t. } B H^{\top}=Z,
\end{equation}
where  $\beta$ is a balance parameter.
To make it easy to solve, we relax the above problem to:
\begin{equation} \label{sobj}
    \min _{Z, B}\left\|H^{\top}-B^{\top} Z\right\|_F^2+\beta\left\|Z\right\|_F^2+\alpha\|B H^{\top}-Z\|^2_F .
\end{equation}
It has two advantages over other anchor-based methods: the efficient and adaptive generation of high-quality anchors. First, existing methods often repeat many times to reduce the uncertainty in results, which is time-consuming and not suitable for large-scale data. Second, existing methods perform anchor selection and graph learning in two separate steps. By contrast, we follow a joint learning approach, where anchors and anchor graphs will be mutually enhanced by each other. 
\indent \\
For a multi-view scenario, each view could contribute differently. Therefore, we introduce learnable weights
$\{\lambda_v\}_{v=1}^V$ and achieve a consensus anchor graph $Z$ by solving the following model:
\begin{equation} \label{mobj}
    \begin{aligned}
    \min _{Z, \{B^v\}_{v=1}^V,\{\lambda_v\}_{v=1}^V}&\sum_{v=1}^V\lambda_v^2(\left\|H^v{}^{\top}-B^v{}^{\top} Z\right\|_F^2 \\
    &+\alpha\|B^v H^v{}^{\top}-Z\|^2_F) +\beta\left\|Z\right\|_F^2,\\
    &\text{s.t.}\space \space \sum_{v=1}^V\lambda_v=1,\space \lambda_v > 0.
    \end{aligned}
\end{equation}
Note that we learn anchors for each view to capture distinctive information. After constructing anchor graph $Z$, $Z^{\top}\Delta Z$ can be used as input to obtain the spectral embedding for clustering in traditional anchor-based methods, where $\Delta$ is a diagonal matrix with $\Delta_{ii}=\sum_{j=1}^N Z_{ji}$. According to \cite{FPMVS}, the right singular vectors of $Z$ are the same as the eigenvectors of $Z^{\top}\Delta Z$. Consequently, we perform singular value decomposition (SVD) on $Z$ and then run $K$-means on the right vector to produce the final result, which needs $\mathcal{O}(m^2N)$ instead of $\mathcal{O}(N^3)$.

\subsection{Optimization} \label{opt}

To solve Eq. (\ref{mobj}), we use an alternative strategy.
\subsubsection{Initialization of $B^v$}
We could optionally initialize $B^v$ with the cluster centers by dividing $H^v$ into $m$ partitions with a $K$-means algorithm. 
\subsubsection{Update $Z$}
By fixing  $\{B^v\}_{v=1}^V$ and $\{\lambda_v\}_{v=1}^V$, we set the derivative of the objective function with respect to $Z$ to zero, we have:
\begin{equation}
    Z = {(1+\alpha)}[\beta I_m +\sum_{v=1}^V\lambda_v^2(B^vB^v{}^{\top}+\alpha I_m)]^{-1}(\sum_{v=1}^V\lambda_v^2B^vH^v{}^{\top})
\end{equation}
For single-view scenario, the solution is $Z=(1+\alpha)(BB^{\top}+(\alpha+\beta)I_m)^{-1}(BH^{\top})$.
\subsubsection{Update $\{B^v\}_{v=1}^V$}
By fixing $Z$ and $\{\lambda_v\}_{v=1}^V$,  Eq. (\ref{mobj}) can be rewritten as: 
\begin{equation}
    R B^v+\beta B^v T^v =C^v,
\end{equation}
where $R=ZZ^{\top}\in \mathbb{R}^{m\times m}$, $T^v=H^v{}^{\top}H^v\in \mathbb{R}^{d_v\times d_v}$, 
$C^v=(1+\alpha)ZH^v\in \mathbb{R}^{m\times d_v}$. Then we can obtain $B^v$ by solving Sylvester Equation.
\subsubsection{Update $\{\lambda_v\}_{v=1}^V$}
Fixing $Z$ and $\{B^v\}_{v=1}^V$, we let $M_v=\left\|H^v{}^{\top}-B^v{}^{\top} Z\right\|_F^2+\alpha\|B^v H^v{}^{\top}-Z\|^2_F$. Then the problem is simplified as: 
\begin{equation}
        \underset{\lambda_v}{\operatorname{min}}\sum_{v=1}^V\lambda_v^2M_v,\quad
        \text{s.t.} \hspace{0.1cm} \sum_{v=1}^ V\lambda_v=1, \lambda_v > 0.
\end{equation}
This is a standard quadratic programming problem, which yields: $ \lambda_v = \frac{\frac{1}{M_v}}{\sum_{p=1}^V\frac{1}{M_p}}.$
   
\textbf{Comment} The optimization procedure will monotonically decrease the objective function value in Eq. (\ref{mobj}) in each iteration. Since the objective function has a lower bound, such as zero, the above iteration converges.
\subsection{Complexity Analysis} \label{CA}
The adjacency graph is often sparse in real-world scenarios. Consequently, we implement graph filtering with sparse matrix techniques, which takes linear time while the multiplication operation takes $\mathcal{O}(f_1N^2)$ in general, where $f_1=\sum_{v=1}^V d_v$. Assume there are  $t$ iterations in total, then the optimization of $Z$ takes $\mathcal{O}(t\operatorname{max}(m^3, mf_1N))$. Specifically, all multiplications and additions take $\mathcal{O}(tmf_1N)$ and the inverse operation needs $\mathcal{O}(tm^3)$. Then optimization of $B^v$ and $\{\lambda_v\}_{v=1}^V$ takes $\mathcal{O}(tf_2)$ and $\mathcal{O}(tmf_1N)$, where $f_2=\sum_{v=1}^V d_v^3$. It is worth pointing out that anchor generation has a constant complexity, which will not be limited by the size of the data. We perform SVD on $Z$ and implement $K$-means to obtain clustering result, which takes $\mathcal{O}(m^2N)$ and $\mathcal{O}(\bar{t}cmN)$ respectively, where $\bar{t}$ is the iteration number of $K$-means and $c$ is cluster number. In practice, $d\ll N$, $m\ll N$, and $t\ll N$, $c$ and $\bar{t}$ are constants, so the proposed method has a linear-time complexity. Moreover, the highest space cost is $m\times N$ or $N \times d$, which means that our approach has a linear space complexity.

We compare our complexity with baselines in Table \ref{tc}. The iteration number $t$ is omitted. The $\widehat{P}$  represents the average degree of the graph in S$^3$GC. $l$ and $K$ are the number of view groups and nearest neighbors for each view group, where \textit{view group} is defined as a group of multiple randomly selected views. $B$ is the batch size, the remaining symbols are the same as those in the main body of CDC. It can be seen that our method has clear advantages and suffers only from the feature dimension. In addition, for high-dimensional data, dimension reduction techniques can be applied.

\begin{table}[htbp]
\label{time2}
\renewcommand\arraystretch{0.9}
  \centering
  \caption{The brief complexity analysis of recent SOTA methods.}
  \label{tc}
  \vspace{8pt}
    \setlength{\tabcolsep}{0.5mm}{
    \begin{tabular}{lccc}
    \toprule
          & Methods & Time  & Space \\
    \midrule
    \multicolumn{1}{c}{\multirow{2}[1]{*}{Single-view}} & MVGRL & $\mathcal{O}(dN^2+d^2N)$ & $\mathcal{O}(N^2+dN)$ \\
          & S$^3$GC & $\mathcal{O}(\widehat{P}dN)$ & $\mathcal{O}(d(B\widehat{P}+N))$ \\
           \midrule
    \multicolumn{1}{c}{\multirow{2}[0]{*}{Multi-view}} & MCGC  & $\mathcal{O}(N^2+dN)$ & $\mathcal{O}(N^2)$ \\
          & MvAGC & $\mathcal{O}(mdN)$ & $\mathcal{O}((m+d)N)$ \\
          \midrule
    \multicolumn{1}{c}{\multirow{2}[0]{*}{Non-graph}} & EOMSC-CA & $\mathcal{O}((m^2+d)N+m^3)$ & $\mathcal{O}((m+d)N)$ \\
          & FastMICE & $\mathcal{O}(lm^{\frac{1}{2}}V^{\frac{1}{2}}N)$ & $\mathcal{O}((c+K+l+V)N)$ \\ 
          \midrule
    \multicolumn{1}{c}{\multirow{1}[1]{*}{Proposed}} & CDC   & $\mathcal{O}((m^2+d)N+d^3)$ & $\mathcal{O}((m+d)N)$ \\
    \bottomrule
    \end{tabular}}
\end{table}

\section{Theoretical Analysis} \label{the}
We establish theoretical support for our method: 1) the filtered features encode node attribute and topology structure; 2) the learned anchor graph is cluster-favorable.

\begin{definition}[\textbf{Grouping effect} \cite{GEKDD}] There are two similar nodes $i$ and $j$ in terms of local topology and node features, i.e., 
    $
    \mathcal{V}_i \rightarrow \mathcal{V}_j \Longleftrightarrow\left(\left\|A_i-A_j\right\|^2 \rightarrow 0\right) \wedge\left(\left\|x_i-x_j\right\|^2 \rightarrow 0\right)
    $,
    the matrix $G$ is said to have a grouping effect if
    $
    \mathcal{V}_i\rightarrow \mathcal{V}_j\Longrightarrow |G_{ip}-G_{jp}|\rightarrow0,\space \forall 1\leq p\leq N.
    $
\end{definition}

\begin{theorem} \label{gf_lemma} Define the distance between filtered node $i$ and $j$ is $\|h_i-h_j\|^2$, we have $\|h_i-h_j\|^2\leq \frac{1}{2^{2k}} [\|(A_i-A_j)\sum_{i=0}^{k-1}{i \choose N}A^iX\|^2+\|x_i-x_j\|^2]$, i.e., the filtered features $H$ preserve both topology and attribute similarity.
\end{theorem}

\begin{proof}
Note $L=I-A$, then $I-\frac{1}{2}L=\frac{A+I}{2}$. Then we have $H=(I-\frac{1}{2}L)^kX=\frac{(A+I)^k}{2^k}X$. Expand it as follows:
\begin{align*}
    H &=\frac{(A+I)^k}{2^k}X=\frac{A\sum_{i=0}^{k-1}{i \choose N}A^i+I}{2^k}X
    \\&=\frac{A\sum_{i=0}^{k-1}{i \choose N}A^{k-i}X+X}{2^k}.
\end{align*}

Then compute the distance of node $i$ and $j$:
\begin{equation}
    \begin{aligned}
        &\|h_i-h_j\|^2\\
        &=\|\frac{(A\sum_{i=0}^{k-1}{i \choose N}A^i X+X)_i-(A\sum_{i=0}^{k-1}{i \choose N}A^i X+X)_j}{2^k}\|^2\\
        &=\frac{1}{2^{2k}}||(A_i-A_j)\sum_{i=0}^{k-1}{i \choose N}A^i X+(X_i-X_j)\|^2\\
        &\leq \frac{1}{2^{2k}} [\|(A_i-A_j)\sum_{i=0}^{k-1}{i \choose N}A^i X\|^2+\|X_i-X_j\|^2]
    \end{aligned}
\end{equation}
So, if $\mathcal{V}_i\rightarrow \mathcal{V}_j$, $\|h_i-h_j\|\rightarrow 0$.  However, when nodes are similar to each other in only one space, 
i.e., either $\|A_i-A_j\|^2\rightarrow 0$ or $\|x_i - x_j\|^2\rightarrow 0$,$\|h_i-h_j\|^2$ has a non-zero upper bound unless $k$ is large enough. This indicates that the filtered representations of similar nodes in both attribute and topology space get closer, and different graph filtering order will adjust this bias.
\begin{theorem}
    \label{lem:grouping1} Let $G=Z^{\top}$, then $|G_{ip}-G_{jp}|^2\leq \|h_i-h_j\|^2\|C_2\|^2_F$, where $C_2$ is a constant matrix. We have $\mathcal{V}_i\rightarrow \mathcal{V}_j$, $|G_{ip}-G_{jp}|^2\rightarrow 0$, i.e.,
    the learned anchor graph $Z$ has a grouping effect.
\end{theorem}
\end{proof}

\begin{proof}
Define $G^*=Z^*{}^{\top}=(1+\alpha)(HB^{\top})(BB^{\top}+(\alpha+\beta)I_m)^{-1}$ and $
\mathcal{L}_i=\|h_i-g_iB\|^2+\alpha\|h_iB^{\top}-g_i\|^2+\beta\|g_i\|^2
$, where $g_i$ is the $i$th row of $G$. Then let $\frac{\partial \mathcal{L}_i}{\partial G_{ip}}|_{g_i=g_i^*}=0$, which yields
 $G_{ip}=\frac{(h_i-g_iB)B_p^{\top}+\alpha h_iB_p^{\top}}{\beta-\alpha}$. Let $C_{1}=(BB^{\top}+(\alpha+\beta)I_m)^{-1}$, thus $g_i=(1+\alpha)h_iB^{\top}C_{1}$. Eventually, 
\begin{equation*}
\begin{aligned}
    G_{ip}&=\frac{[h_i-(1+\alpha)h_iB^{\top}C_{1}B]B_p^{\top}+\alpha h_iB_p^{\top}}{\beta-\alpha}\\
    &=\frac{h_i(1+\alpha)(I_d-B^{\top}C_{1}B)B_p^{\top}}{\beta-\alpha}.
\end{aligned}
\end{equation*}

Remarking $C_2=\frac{(1+\alpha)(I_d-B^{\top}C_{1}B)B_p^{\top}}{\beta-\alpha}$, we obtain $|G_{ip}-G_{jp}|^2\leq \|h_i-h_j\|^2\|C_2\|^2$.
\end{proof}

This indicates that local structures of similar nodes tend to be identical on the learned graph $Z$, which makes corresponding nodes be clustered into the same group. In other words, the learned graph is clustering-friendly. To intuitively demonstrate the grouping effect of the anchor graph, we plot five diagrams of $Z$ in Fig. \ref{visual} and apply different colors to represent the different values. Since the used data sets contain too many nodes, we randomly select 20 nodes of each class for clear visualization and set the anchor number $m=3$, i.e., $Z\in \mathbbm{R}^{60\times 3}$. In these figures, we have reordered all indices of selected nodes and shown them on the x-axis; that is, the 1st-20th, 21st-40th and 41st-60th nodes are of three different classes. The values of $Z$ are on the y-axis. It's clear that nodes from the same group tend to have similar values, which in turn leads to better clustering results. Moreover, $Z$ achieving higher clustering performance has a stronger grouping effect; for example, $Z$ on ACM has a stronger grouping effect than the one on Pubmed (details on ACM and Pubmed can be found in Section \ref{exps}).
       \begin{figure}
           \centering
           \subfigure[ACM(ACC=93.6\%)]{
           \includegraphics[width=1.0\linewidth]{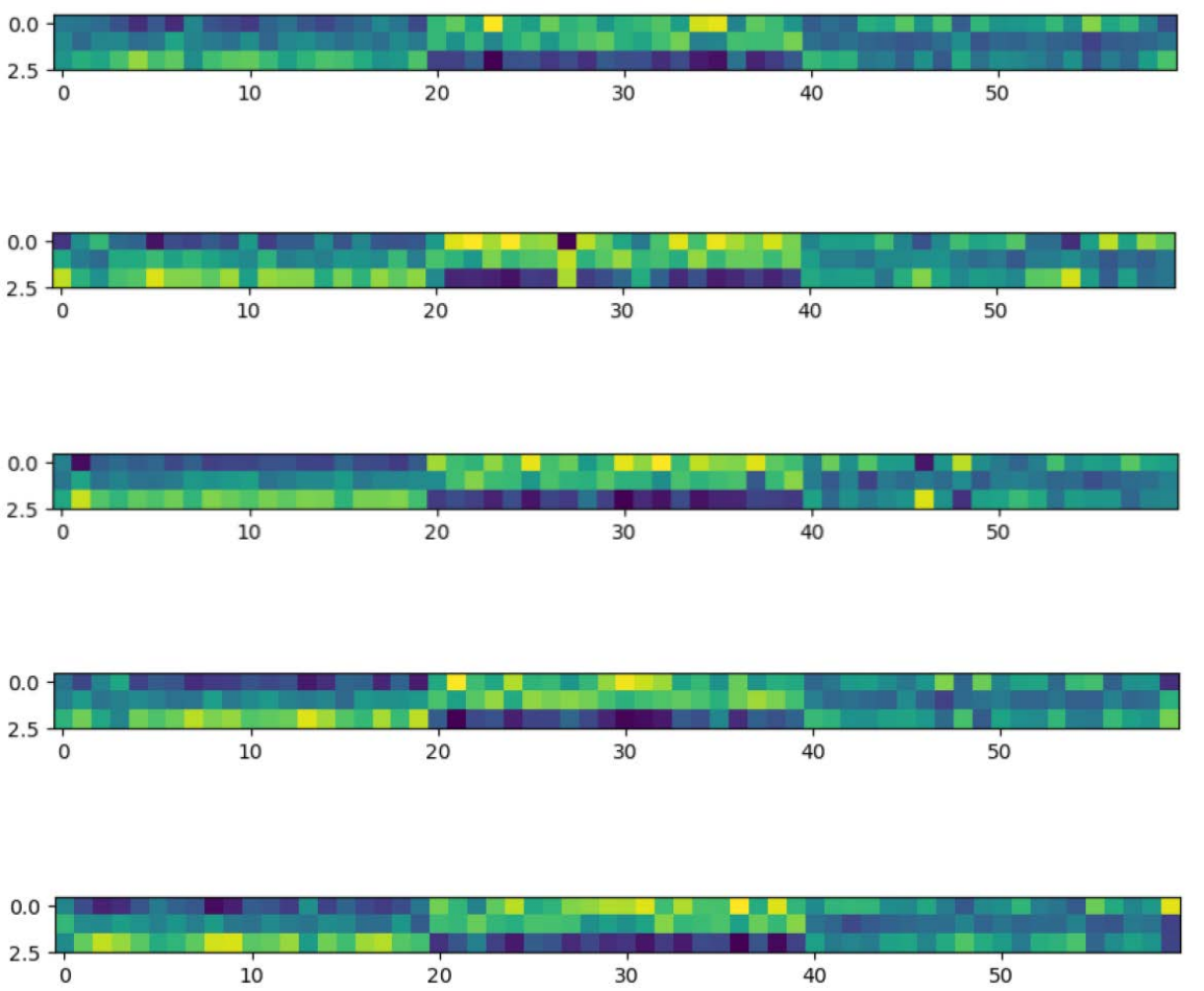}
           }
           \subfigure[Pubmed(ACC=74.1\%)]{
           \includegraphics[width=1.\linewidth]{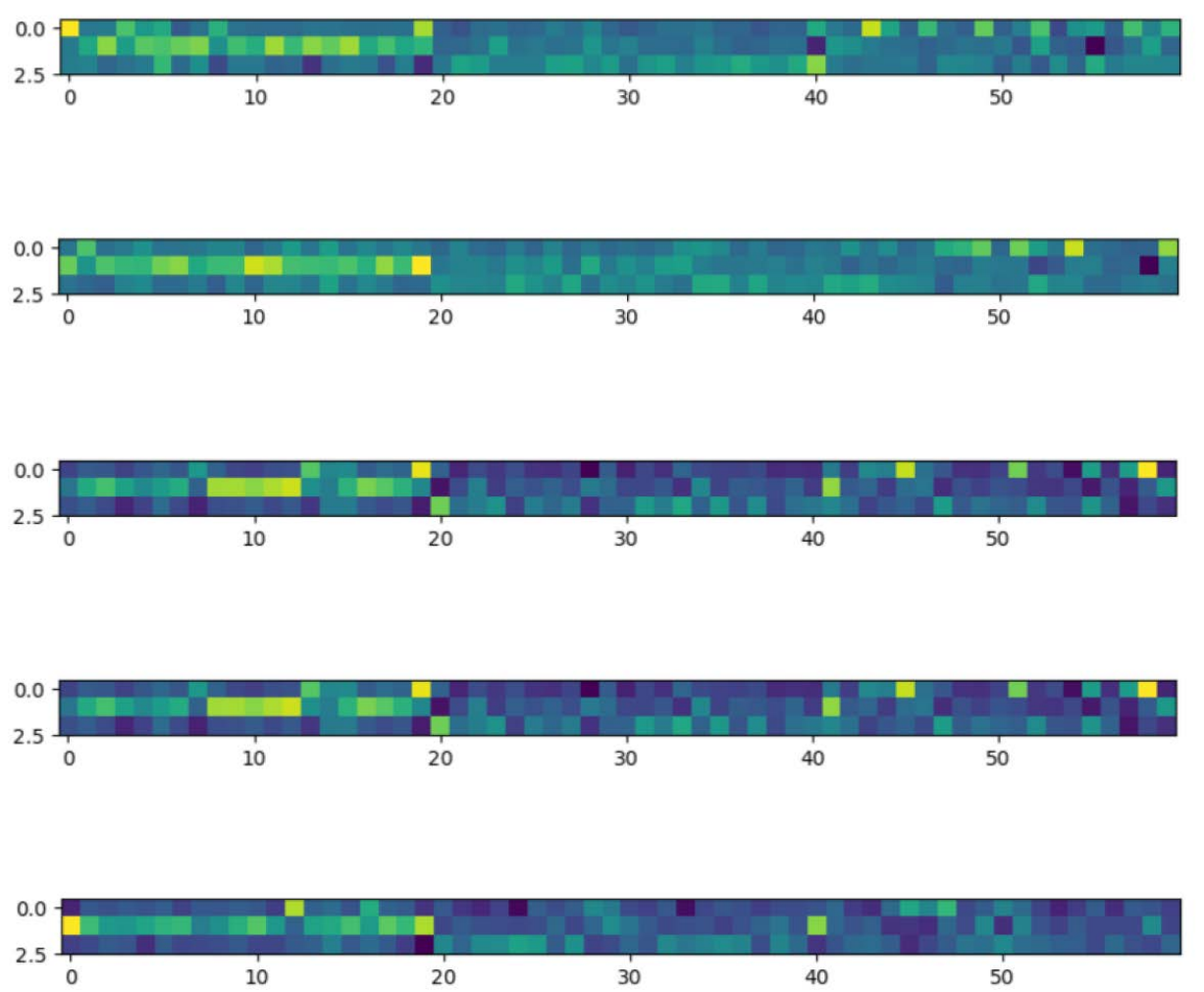}
           }
           \caption{Visualization of learned graph $Z$'s grouping effect.}
           \label{visual}
       \end{figure}

\section{Experiments} \label{exps}
\subsection{Datasets and Metrics}
\begin{table*}[htbp]
  \centering
  \small
  \caption{Statistical information of datasets.}
    \setlength\tabcolsep{0.8pt}
    \begin{tabular}{cccccc}
    \toprule
    \multicolumn{2}{c}{Type} & Datasets & Samples & Edges/Dims & Clusters \\
    \midrule
    \multirow{8}[8]{*}{Graph} & \multirow{2}[2]{*}{Single-view} & Citeseer & 3327   & 4614 / 3703 & 6 \\
          &       & Pubmed & 19717   & 44325 / 500 & 3 \\
\cmidrule{2-6}          & \multirow{2}[2]{*}{Multi-relational} & ACM   & 3025    & 29281, 2210761 / 1830 & 3 \\
          &       & DBLP  & 4057   & 11113, 5000495, 6776335 / 334 & 4 \\
\cmidrule{2-6}          & \multirow{2}[2]{*}{Multi-attribute} & AMAP  & 7487  & 119043 / 745, 7487 & 8 \\
          &       & AMAC  & 13381  & 245778 / 767, 13381 & 10 \\
\cmidrule{2-6}          & \multirow{2}[2]{*}{Extra-/Large-scale} & Products & 2449029   & 61859140 / 100 & 47 \\
          &       & Papers100M & 111059956  & 1615685872 / 128 & 172 \\
    \midrule
    \multirow{2}[2]{*}{Non-graph} & \multirow{2}[2]{*}{Large-scale multi-view} & YTF-31 & 101499   & 507495 / 64, 512, 64, 647, 838 & 31 \\
          &       & YTF-400 & 398191& 1990955 / 944, 576, 512, 640 & 400 \\
    \bottomrule
    \end{tabular}%
  \label{datasets}%
\end{table*}%

To show the effectiveness and efficiency of the CDC, we evaluate CDC on 10 benchmark datasets, including 6 multi-view data and 4 single-view data. More specifically, ACM and DBLP \cite{O2MAC} are multi-view graphs with multiple relations, AMAP and AMAC \cite{MvAGC} are multi-attribute graphs, YTF-31 \cite{EOMSC-CA} and YTF-400 \cite{FastMICE} are multi-view non-graph data (YouTube-Faces); Citeseer, Pubmed \cite{GCN}, Products and Papers100M \cite{GraphSAGE} are single-view graphs, where the latter two are from Open Graph Benchmarks \cite{hu2020open}.  The statistical information for these datasets is shown in Table \ref{datasets}. Most notably, YTF-400 represents the largest multi-view non-graph dataset, while Papers100M is the largest graph used in the clustering task. We adopt four popular clustering metrics, including ACCuracy (ACC), Normalized Mutual Information (NMI), F1 score, and Adjusted Rand Index (ARI). ACC measures the degree of match between clustering results and ground truth labels. NMI assesses the similarity between two clusters, providing a score of perfect agreement. The F1 score balances precision and recall in clustering. ARI quantifies cluster similarity, adjusting for chance. A higher value of them indicates better performance.

\begin{table*}[]
\centering
\caption{The parameter ranges for baselines.}
\begin{tabular}{llllllll}
\hline
&Methods   & Learning rate  & Filter order & Loss trade-off     & Others                                    \\ \hline

    \multicolumn{1}{c}{\multirow{1}[1]{*}{Single-view}} &CCGC      & 1e-2,1e-3,1e-4 & 1,2,3,4,5     & 1e-2,1e-1,...,1e2  & $\tau$ = 0.5,0.55,0.6,0.65,0.7            \\ \hline
    \multicolumn{1}{c}{\multirow{2}[1]{*}{Multi-relational }} &BTGF      & 1e-2,1e-3      & 1,2,3,4,5     &                    & $\gamma$ = 0.1,1,10,100,1000              \\
&DMG       & 1e-3,1e-4      & 1,2,3         & 1e-3,1e-2,...,1e3  & $\tau$ = 1,2                              \\ \hline
\multicolumn{1}{c}{\multirow{3}[1]{*}{Non-graph multi-view}} &FastMICE  &                & 1,2,3,4,5     & 1e-2,1e-1,...,1e2  &                                           \\ 
&MSGL  &                & 1,2,3,4,5     & 1e-2,1e-1,...,1e2  &                                           \\
&LMVSC &                & 1,2,3,4,5     & 1e-2,1e-1,...,1e2  &                                           \\ \hline
\multicolumn{1}{c}{\multirow{5}[1]{*}{Multi-attribute}} &MCGC      & 1e-2,1e-3,1e-4 & 1,2,3,4,5     & 1e-3,1e-2,...,1e-3 &                                           \\ 
&MvAGC     &                & 1,2,3,4,5     & 1,10,40,70,100     & $\gamma$ = 3,4,5,6,7;P = 1,2,3            \\
&MAGCN     & 1e-2,1e-3,1e-4 & 1,2,3         & 1e-2,1e-1,...,1e2  &                                           \\ 
&MVGRL     & 1e-2,1e-3      & 2,4,8,12      & 1e-3,1e-2,...,1e3  & $\alpha$=2, t=5                           \\ 
&COMPLETER & 1e-3,1e-4           &               & 1e-2,1e-1,...,1e2  & $\alpha$ = 1,3,5,7,9 $\eta$ = 0.3,0.7,0.9 \\ \hline
\end{tabular}
\label{123}
\end{table*}

\subsection{Experimental Setup}

We compare CDC\footnote{https://github.com/XieXuanting} with a number of single-view methods as well as multi-view methods.\\
\textbf{Single-view graph}
Baselines include METIS \cite{METIS}, Node2vec \cite{Node2vec}, DGI \cite{DGI}, DMoN \cite{DMoN}, GRACE \cite{GRACE}, BGRL \cite{BGRL}, MVGRL \cite{MVGRL}, S$^3$GC  \cite{S3GC} \textcolor{black}{, and CCGC \cite{CCGC}}.

METIS uses only structural information to partition graphs. Node2vec is a well-known graph embedding algorithm based on random walks. DMoN integrates spectral clustering with graph neural networks. DGI learns node representations by maximizing mutual information between patch representations and the corresponding high-level summaries of graphs. GRACE, BGRL, and MVGRL are three contrastive graph representation learning methods. S$^3$GC is a recent scalable graph clustering method, which uses a light encoder and a random walk-based sampler. CCGC \cite{CCGC} develops a new method for constructing positive and negative samples.\\

\textbf{Multi-view graph}
There are 11 baselines on multi-view graphs clustering, including DAEGC \cite{DAEGC}, O2MAC \cite{O2MAC}, HDMI \cite{HDMI}, CMGEC \cite{CMGEC}, COMPLETER \cite{COMPLETER}, MvAGC \cite{MvAGC}, MCGC \cite{MCGC}, MVGRL \cite{MVGRL}, \textcolor{black}{ BTGF \cite{qian2024upper}, DMG \cite{mo2023disentangled},} and MAGCN \cite{MAGCN}. The first five methods are only applicable to data with multiple graphs or multiple attributes, whereas the last six are applicable to general multi-view graph data.\\
Graph attention auto-encoders and GCNs are used in DAEGC. 
To get consistent embeddings, CMGEC adds a graph fusion network to multiple graph auto-encoders. In O2MAC, the most informative view is selected to learn cluster representations. HDMI learns node embeddings by using high-order mutual information. MAGCN applies the graph auto-encoder on both attributes and topological graphs to learn consensus representations. Through contrastive mechanisms, COMPLETER and MVGRL learn a common representation shared across multiple views and graphs. MCGC uses a contrastive regularizer to boost the quality of the learned graph. In MvAGC, high-order topological interactions are explored to improve clustering performance. BTGF \cite{qian2024upper} develops a novel filter to account for the correlations between graphs. DMG \cite{mo2023disentangled} aims to extract common and private information extraction.\\

\textbf{Non-graph}
We compare CDC with six scalable MVC methods on non-graph data, including BMVC \cite{BMVC}, LMVSC \cite{LMVSC}, MSGL \cite{MSGL}, FPMVS \cite{FPMVS}, EOMSC-CA \cite{EOMSC-CA}, and FastMICE \cite{FastMICE}.

BMVC learns discrete representations and binary cluster structures jointly to integrate collaborative information. MSGL and LMVSC are two scalable subspace clustering methods. FPMVS and EMOMSC-CA are two adaptive anchor-based algorithms. The differences between CDC and these methods are: 1) CDC uses a similarity-preservation regularizer while anchor matrices are assumed to be unitary matrices in FPMVS and EMOMSC; 2) the complexity of anchor generation in CDC is not linked with data size. FastMICE constructs anchor graphs using features, anchors, and neighbors jointly.\\

\textbf{Parameter setting}
We perform grid search to find the best parameters. The balance parameters $\alpha$ and $\beta$ are set as $[1e-3,1,1e1, 1e3, 1e4]$. The number of anchors $m$ is set as $[c, 10,30, 50, 70, 100]$. For comparison methods, we either directly quote the results from the original authors or conduct a grid search using the suggested range from the original paper. The search ranges are summarized in Table \ref{123}. All experiments are conducted on the same machine with the Intel(R) Core(TM) i9-12900k CPU, two GeForce GTX 3090 GPUs, and 128GB RAM.


\subsection{Results}
                      
\begin{table}[H]
    \centering
    \caption{Results on extra-large-scale graph.}
          \setlength{\tabcolsep}{1.mm}
            \begin{tabular}{cccccc}
            \toprule
            \multirow{2}[4]{*}{Metrics} & \multicolumn{5}{c}{Papers100M} \\
        \cmidrule{2-6}          & $K$-means & Node2vec & DGI   & S$^3$GC   & CDC \\
            \midrule
            ACC   & 0.144  & \textcolor[rgb]{ 1,  0,  0}{\textbf{0.175 }} & 0.151  & 0.173  & \textcolor[rgb]{ .267,  .447,  .769}{\textbf{0.174 }} \\
            NMI   & 0.368  & 0.380  & 0.416  & \textcolor[rgb]{ 1,  0,  0}{\textbf{0.453 }} & \textcolor[rgb]{ .267,  .447,  .769}{\textbf{0.427 }} \\
            ARI   & 0.074  & \textcolor[rgb]{ .267,  .447,  .769}{\textbf{0.112 }} & 0.096  & 0.110  & \textcolor[rgb]{ 1,  0,  0}{\textbf{0.114 }} \\
            F1    & 0.101  & 0.099  & 0.111  & \textcolor[rgb]{ .267,  .447,  .769}{\textbf{0.118 }} & \textcolor[rgb]{ 1,  0,  0}{\textbf{0.119 }} \\
            \bottomrule
            \end{tabular}%
          \label{s2}%
\end{table}%
\subsubsection{Single-view Scenario}

The results on the small-scale graph Citeseer, medium-scale graph Pubmed, large-scale graph Products, and extra-large-scale graph Papers100M are shown in Table \ref{s2} and Table \ref{s1}. Among those 16 statistics, CDC achieves the best performance in 9 cases and the second best performance in 5 cases. Note that most neural network-based methods cannot handle large and extra-large-scale graphs. In Citeseer and Pubmed, our method achieves the best results, and in Products and Papers100M, our method produces competitive results. In particular, on Pubmed, CDC surpasses the most recent S$^3$GC method by more than 3$\%$ in all metrics. CDC also shows a slight advantage against S$^3$GC on the largest Papers100M dataset. CDC is inferior to S$^3$GC on Products, which may be because low-pass filtering is not enough on this dataset. Furthermore, CDC involves a lower time cost in comparison to S$^3$GC. Specifically, it takes $\sim$5mins and $\sim$4hs on Products and Papers100M, while S$^3$GC consumes $\sim$1h and $\sim$24hs, demonstrating CDC's efficiency when it comes to large/extra-large-scale graphs. With respect to many GNN-based methods, like DGI, DMON, GRACE, BGRL, and MVGRL, CDC demonstrates a clear edge.
\begin{table*}[htbp]
  \centering
  \caption{Results on single-view graphs. "-" denotes that the method ran out
of memory (OM) or didn't converge. The best results are indicated in \textbf{\textcolor{red}{red}} while the second-best results are highlighted in \textcolor[rgb]{ .267,  .447,  .769}{\textbf{blue}}.}
  \setlength\tabcolsep{1.5pt}
  
  \begin{tabular}{p{6em}cccccccccccc}
    \toprule
    \multicolumn{1}{c}{} & \multicolumn{4}{c}{Citeseer}  & \multicolumn{4}{c}{Pubmed}    & \multicolumn{4}{c}{Products} \\
    \midrule
    \multicolumn{1}{c}{Method} & ACC   & NMI   & ARI   & F1    & ACC   & NMI   & ARI   & F1    & ACC   & NMI   & ARI   & F1 \\
    \midrule
    METIS & 0.413  & 0.170  & 0.150  & 0.400   & 0.693 & 0.297 & 0.323 & 0.682 & 0.294 & 0.468 & 0.220  & 0.145 \\
    Node2vec & 0.421  & 0.240  & 0.116 & 0.401 & 0.641 & 0.288 & 0.258 & 0.634 & 0.357 & \textcolor[rgb]{ .267,  .447,  .769}{\textbf{0.489}} & \textcolor[rgb]{ .267,  .447,  .769}{\textbf{0.247}} & 0.170 \\
    DGI   & 0.686 & 0.435 & 0.445 & 0.643 & 0.657 & 0.322 & 0.292 & 0.654 & 0.320  & 0.467 & 0.192 & 0.174 \\
    DMoN  & 0.385 & 0.303  & 0.200   & 0.437 & 0.351 & 0.257 & 0.108 & 0.343 & 0.304 & 0.428 & 0.210  & 0.139 \\
    GRACE & 0.631 & 0.399 & 0.377 & 0.603 & 0.637 & 0.308 & 0.276 & 0.628 & - & - & - & - \\
    BGRL  & 0.675  & 0.422 & 0.428 & 0.631 & 0.654 & 0.315 & 0.285 & 0.649 & - & - & - & - \\
    MVGRL & \textcolor[rgb]{ .267,  .447,  .769}{\textbf{0.703}} & \textcolor[rgb]{ 1,  0,  0}{\textbf{0.459}} & \textcolor[rgb]{ 1,  0,  0}{\textbf{0.471}} & \textcolor[rgb]{ .267,  .447,  .769}{\textbf{0.654}} & 0.675 & \textcolor[rgb]{ .267,  .447,  .769}{\textbf{0.345}} & 0.310  & 0.672 & - & - & - & - \\
    S$^3$GC   & 0.688 & 0.441 & 0.448 & 0.643 & \textcolor[rgb]{ .267,  .447,  .769}{\textbf{0.713}} & 0.333 & \textcolor[rgb]{ .267,  .447,  .769}{\textbf{0.345}} & \textcolor[rgb]{ .267,  .447,  .769}{\textbf{0.703}} & \textcolor[rgb]{ 1,  0,  0}{\textbf{0.402}}  & \textcolor[rgb]{ 1,  0,  0}{\textbf{0.536}} & \textcolor[rgb]{ 1,  0,  0}{\textbf{0.25}} & \textcolor[rgb]{ 1,  0,  0}{\textbf{0.23}} \\

    \textcolor{black}{CCGC} &0.698 &0.443 &0.457 &0.627 &0.681 &0.309 &0.322&0.682&-&-&-&-\\
    CDC  & \textcolor[rgb]{ 1,  0,  0}{\textbf{0.709 }} & \textcolor[rgb]{ .267,  .447,  .769}{\textbf{0.444 }} & \textcolor[rgb]{ 1,  0,  0}{\textbf{0.471 }} & \textcolor[rgb]{ 1,  0,  0}{\textbf{0.661 }} & \textcolor[rgb]{ 1,  0,  0}{\textbf{0.741 }} & \textcolor[rgb]{ 1,  0,  0}{\textbf{0.371 }} & \textcolor[rgb]{ 1,  0,  0}{\textbf{0.383 }} & \textcolor[rgb]{ 1,  0,  0}{\textbf{0.737 }} & \textcolor[rgb]{ .267,  .447,  .769}{\textbf{0.366 }} & 0.390  & 0.121  & \textcolor[rgb]{ .267,  .447,  .769}{\textbf{0.187 }} \\
    \bottomrule
    \end{tabular}%
  \label{s1}%
  
\end{table*}%

\subsubsection{Multi-view Scenario}

\paragraph{Clustering on Multi-view Graphs}

The CDC clustering results on multi-attribute and multi-relational graphs are reported in Table \ref{mv1} and Table \ref{mv2}. The performance of CDC is much better than that of any other method across four benchmarks and all metrics. For example, compared to the second-best method, MCGC, the ACC, NMI, and ARI metrics on ACM, AMAP, and AMAC are improved by more than 5$\%$, 7$\%$, and 9$\%$ on average, respectively. Although MvAGC samples nodes as anchors, it takes more time than CDC since its sampling strategy suffers from low efficiency. Specifically, CDC is more than $2\times$ and  $5\times$ faster on the multi-relational and multi-attribute graphs, respectively. Clustering on Multi-view Graphs shows that the advantage over other methods is more significant. Therefore, CDC is a promising clustering method for graph data in various forms. CDC performs excellently in multi-view datasets but is unstable in the single-view datasets. This is because the proposed learnable weights $\{\lambda_v\}_{v=1}^V$  capture more information in multi-view graphs.

\begin{table*}[htbp]
\renewcommand\arraystretch{0.9}
  \centering
  \caption{Results on the multi-relational graph.}
    \begin{tabular}{ccccccccc}
    \toprule
    \multicolumn{1}{c}{\multirow{2}[4]{*}{  Method}} & \multicolumn{4}{c}{ACM}       & \multicolumn{4}{c}{DBLP} \\
\cmidrule{2-9}          & ACC   & NMI   & ARI   & F1    & ACC   & NMI   & ARI   & F1 \\
    \midrule
    DAEGC  & 0.891  & 0.643  & 0.705  & 0.891  & 0.873  & 0.674  & 0.701  & 0.862  \\
    O2MAC  & 0.904  & 0.692  & 0.739  & 0.905  & 0.907  & 0.729  & 0.778  & 0.901  \\
    HDMI  & 0.874  & 0.645  & 0.674  & 0.872  & 0.885  & 0.692  & 0.753  & 0.865  \\
    CMGEC & 0.909  & 0.691  & 0.723  & 0.907  & 0.910  & 0.724  & 0.786  & 0.904  \\
    MvAGC  & 0.898  & 0.674  & 0.721  & 0.899  & 0.928  & 0.773  & 0.828  & 0.923  \\
    MCGC  & \textcolor[rgb]{ .267,  .447,  .769}{\textbf{0.915 }} & \textcolor[rgb]{ .267,  .447,  .769}{\textbf{0.713 }} & \textcolor[rgb]{ .267,  .447,  .769}{\textbf{0.763 }} & \textcolor[rgb]{ .267,  .447,  .769}{\textbf{0.916 }} & \textcolor[rgb]{ .267,  .447,  .769}{\textbf{0.930 }} & \textcolor[rgb]{ .267,  .447,  .769}{\textbf{0.775 }} & \textcolor[rgb]{ .267,  .447,  .769}{\textbf{0.830 }} & \textcolor[rgb]{ .267,  .447,  .769}{\textbf{0.925 }} \\

    \textcolor{black}{DMG} &0.871 &0.641 &0.672 &0.873 &0.885 &0.690 &0.731 &0.879 \\

    \textcolor{black}{BTGF} &0.901 &0.689 &0.731 &0.901  &0.881 &0.663 &0.725 &0.873 \\
    CDC  & \textcolor[rgb]{ 1,  0,  0}{\textbf{0.936 }} & \textcolor[rgb]{ 1,  0,  0}{\textbf{0.769 }} & \textcolor[rgb]{ 1,  0,  0}{\textbf{0.817 }} & \textcolor[rgb]{ 1,  0,  0}{\textbf{0.936 }} & \textcolor[rgb]{ 1,  0,  0}{\textbf{0.933 }} & \textcolor[rgb]{ 1,  0,  0}{\textbf{0.781 }} & \textcolor[rgb]{ 1,  0,  0}{\textbf{0.836 }} & \textcolor[rgb]{ 1,  0,  0}{\textbf{0.929 }} \\
    \bottomrule
    \end{tabular}%
  \label{mv1}%
\end{table*}
\begin{table*}[t]
\renewcommand\arraystretch{1.0}
  \centering
  \caption{Results on the multi-attribute graph.}
    \begin{tabular}{ccccccccc}
    \toprule
    Datasets & \multicolumn{4}{c}{AMAP} & \multicolumn{4}{c}{AMAC} \\
    \midrule
    Method & ACC   & NMI   & ARI   & F1    & ACC   & NMI   & ARI   & F1 \\
    \midrule
    COMPLETER & 0.368  & 0.261  & 0.076  & 0.307  & 0.242  & 0.156  & 0.054  & 0.160  \\
    MVGRL & 0.505  & 0.433  & 0.238  & 0.460  & 0.245  & 0.101  & 0.055  & 0.171  \\
    MAGCN & 0.517  & 0.390  & 0.240  & 0.474  & -    & -    & -    & - \\
    MvAGC & 0.678  & 0.524  & 0.397  & 0.640  & 0.580  & 0.396  & 0.322  & 0.412  \\
    MCGC  & \textcolor[rgb]{ .267,  .447,  .769}{\textbf{0.716 }} & \textcolor[rgb]{ .267,  .447,  .769}{\textbf{0.615 }} & \textcolor[rgb]{ .267,  .447,  .769}{\textbf{0.432 }} & \textcolor[rgb]{ .267,  .447,  .769}{\textbf{0.686 }} & \textcolor[rgb]{ .267,  .447,  .769}{\textbf{0.597 }} & \textcolor[rgb]{ .267,  .447,  .769}{\textbf{0.532 }} & \textcolor[rgb]{ .267,  .447,  .769}{\textbf{0.390 }} & \textcolor[rgb]{ .267,  .447,  .769}{\textbf{0.520 }} \\
    CDC  & \textcolor[rgb]{ 1,  0,  0}{\textbf{0.795 }} & \textcolor[rgb]{ 1,  0,  0}{\textbf{0.707 }} & \textcolor[rgb]{ 1,  0,  0}{\textbf{0.620 }} & \textcolor[rgb]{ 1,  0,  0}{\textbf{0.730 }} & \textcolor[rgb]{ 1,  0,  0}{\textbf{0.647 }} & \textcolor[rgb]{ 1,  0,  0}{\textbf{0.604 }} & \textcolor[rgb]{ 1,  0,  0}{\textbf{0.437 }} & \textcolor[rgb]{ 1,  0,  0}{\textbf{0.546 }} \\
    \bottomrule
    \end{tabular}%
  \label{mv2}%
\end{table*}%

\paragraph{Clustering on Multi-view Non-graph data}
            \begin{table}[H]
                    \centering
                    \footnotesize
                    
                      \caption{Results on large-scale multi-view non-graph data.}
   \setlength{\tabcolsep}{0.4mm}{
    
    \begin{tabular}{ccccccc}
    \toprule
    \multirow{2}[4]{*}{Method} & \multicolumn{3}{c}{YTF-31} & \multicolumn{3}{c}{YTF-400} \\
\cmidrule{2-7}          & ACC   & NMI   & F1    & ACC   & NMI   & F1 \\
    \midrule
    BMVC  & 0.090  & 0.059  & 0.058  & -   & -   & - \\
    LMVSC & 0.140  & 0.118  & 0.083  & 0.489  & 0.767  & 0.589  \\
    MSGL  & 0.167  & 0.001  & 0.151  & 0.502  & 0.738  & \textcolor[rgb]{ 1,  0,  0}{\textbf{0.606 }} \\
    FPMVS & 0.230  & 0.234  & 0.140  & 0.562  & \textcolor[rgb]{ .267,  .447,  .769}{\textbf{0.797 }} & 0.472  \\
    EOMSC-CA & 0.265  & 0.003  & 0.164  & 0.570  & 0.779  & 0.408  \\
    FastMICE & \textcolor[rgb]{ .267,  .447,  .769}{\textbf{0.275 }} & \textcolor[rgb]{ .267,  .447,  .769}{\textbf{0.236 }} & \textcolor[rgb]{ .267,  .447,  .769}{\textbf{0.295 }} & \textcolor[rgb]{ .267,  .447,  .769}{\textbf{0.564 }} & \textcolor[rgb]{ 1,  0,  0}{\textbf{0.798 }} & 0.509 \\
    CDC   & \textcolor[rgb]{ 1,  0,  0}{\textbf{0.285 }} & \textcolor[rgb]{ 1,  0,  0}{\textbf{0.260 }} & \textcolor[rgb]{ 1,  0,  0}{\textbf{0.298 }} & \textcolor[rgb]{ 1,  0,  0}{\textbf{0.571 }} & 0.745  & \textcolor[rgb]{ .267,  .447,  .769}{\textbf{0.591 }} \\
    \bottomrule
    \end{tabular}}
     \label{nongraph}
\end{table}
To process non-graph data, we manually construct 5-nearest neighbor graphs for graph filtering. Table \ref{nongraph} shows the clustering results on YTF-31 and YTF-400. We find that most existing methods can't handle YTF-400, which is the largest non-graph multi-view data. CDC still achieves the best results in most cases. Although other methods also use anchor-based ideas, their computation time is still high. Specifically, CDC takes $\sim$20s and $\sim$1min, respectively, while EOMSC-CA requires $\sim$2mins and $\sim$6mins, and FastMICE takes $\sim$30s and $\sim$3mins on these two datasets. This verifies that CDC is also a promising clustering method for non-graph data. The time cost of several recent SOTA methods is summarized in Fig. \ref{time}.

\begin{figure*}[t]
    \centering
    \includegraphics[width=1.\linewidth]{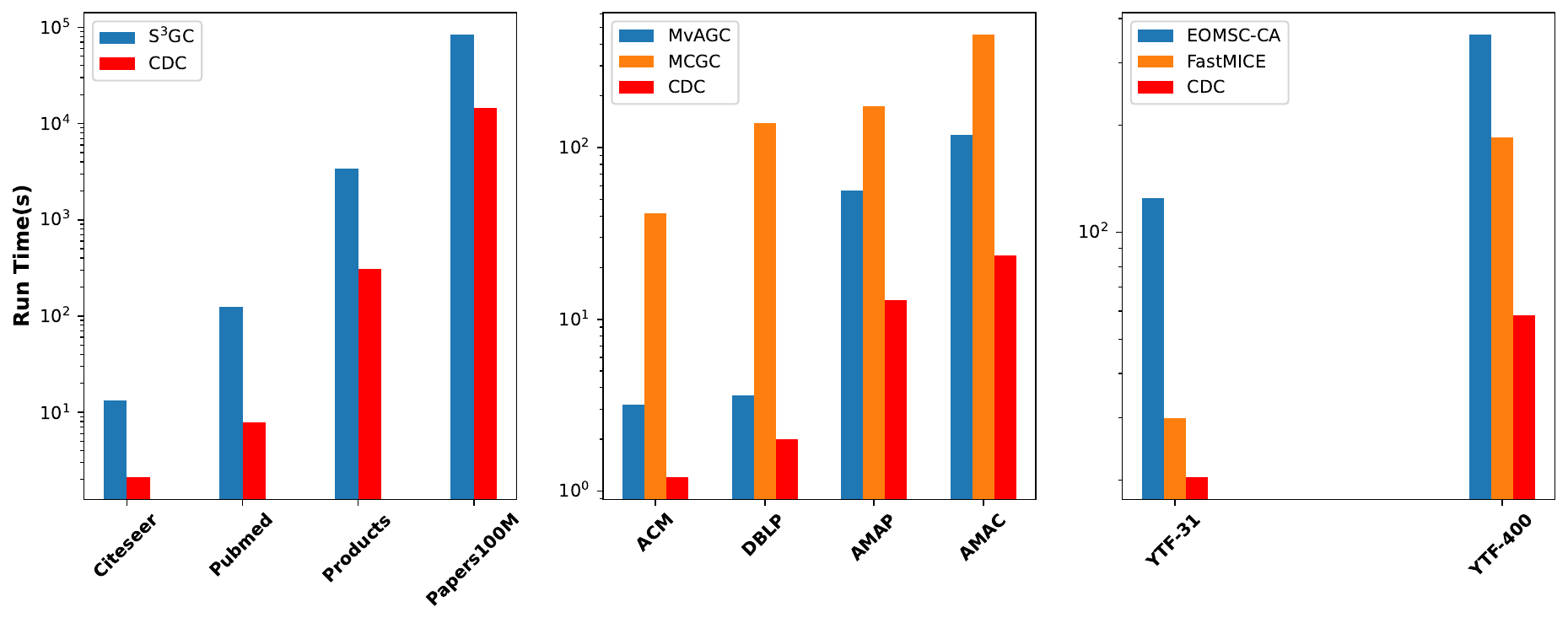}
    \caption{Run time of existing SOTA methods on various datasets}
    \label{time}
    \vspace{-10pt}
\end{figure*}


\subsection{Ablation Study}

\subsubsection{Effect of Similarity-Preserving}
Anchors are generated adaptively in the similarity space constrained by a similarity-preserving (marked as SP) regularizer. To clearly show the effect of SP, we remove it from the model and test the performance of CDC \textit{w/o} SP on Pubmed and ACM in Table \ref{ablation}. It's clear that similarity-preserving does improve the clustering performance by 3$\%$ on average. Moreover, CDC takes less time than CDC \textit{w/o} SP on two datasets. The reason is that the computation complexity for $B^v=(ZZ^{\top})^{-1}(ZH^v)$ is $\mathcal{O}(md_vN)$, which is higher than $\mathcal{O}(d_v^3)$ in CDC. Therefore, as a bonus, the SP regularizer helps reduce the complexity of anchor generation. Moreover, it improves the quality of anchors. As observed in Fig. \ref{PA2}, CDC achieves the best results with a few anchors, which further reduces the computation cost. In fact, too many anchors could deteriorate performance since some noisy anchors, non-representative anchors could be introduced.
\begin{table}[t]
\renewcommand\arraystretch{0.8}
  \centering
  \caption{Results of CDC with/without GF and SR.}
  \setlength\tabcolsep{1pt}
      \begin{tabular}{ccccccccc}
    \toprule
    \multirow{2}[4]{*}{Method} & \multicolumn{4}{c}{Pubmed}    & \multicolumn{4}{c}{ACM} \\
\cmidrule{2-9}          & ACC   & NMI   & F1    & Time (s) & ACC   & NMI   & F1    & Times (s) \\
    \midrule
    CDC  & 0.741  & 0.371  & 0.737  & 2.03  & 0.936  & 0.769  & 0.936  & 0.81  \\
    \midrule
    \multirow{2}[2]{*}{\textit{w/o} SP} & 0.707  & 0.349  & 0.704  & 5.91  & 0.918  & 0.710  & 0.919  & 1.03  \\
          & \textcolor[rgb]{ 0,  .69,  .314}{\textbf{(-0.034)}} & \textcolor[rgb]{ 0,  .69,  .314}{\textbf{(-0.022)}} & \textcolor[rgb]{ 0,  .69,  .314}{\textbf{(-0.033)}} & \textcolor[rgb]{ 0,  .69,  .314}{\textbf{(+3.88)}} & \textcolor[rgb]{ 0,  .69,  .314}{\textbf{(-0.018)}} & \textcolor[rgb]{ 0,  .69,  .314}{\textbf{(-0.059)}} & \textcolor[rgb]{ 0,  .69,  .314}{\textbf{(-0.017)}} & \textcolor[rgb]{ 0,  .69,  .314}{\textbf{(+0.22)}} \\
    \midrule
    \multirow{2}[2]{*}{\textit{w/o} GF} & 0.626  & 0.256  & 0.639  & 2.86  & 0.872  & 0.585  & 0.871  & 0.88  \\
          & \textcolor[rgb]{ 0,  .69,  .314}{\textbf{(-0.115)}} & \textcolor[rgb]{ 0,  .69,  .314}{\textbf{(-0.115)}} & \textcolor[rgb]{ 0,  .69,  .314}{\textbf{(-0.098)}} & \textcolor[rgb]{ 0,  .69,  .314}{\textbf{(+0.83)}} & \textcolor[rgb]{ 0,  .69,  .314}{\textbf{(-0.064)}} & \textcolor[rgb]{ 0,  .69,  .314}{\textbf{(-0.184)}} & \textcolor[rgb]{ 0,  .69,  .314}{\textbf{(-0.055)}} & \textcolor[rgb]{ 0,  .69,  .314}{\textbf{(+0.07)}} \\
    \bottomrule
    \end{tabular}%
  \label{ablation}%
\end{table}%
\begin{figure}[htbp]
 \vspace{-10pt}
		\centering
		\subfigure[ACM]{
			\includegraphics[width=1.\linewidth]{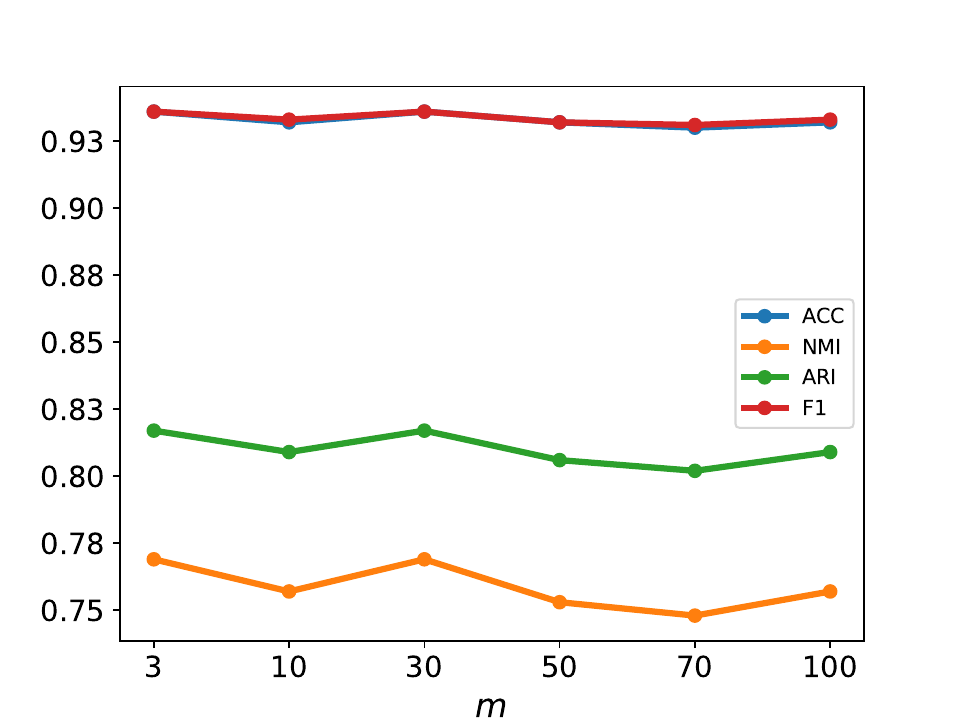}
		}
		\subfigure[Pubmed]{
			\includegraphics[width=1.\linewidth]{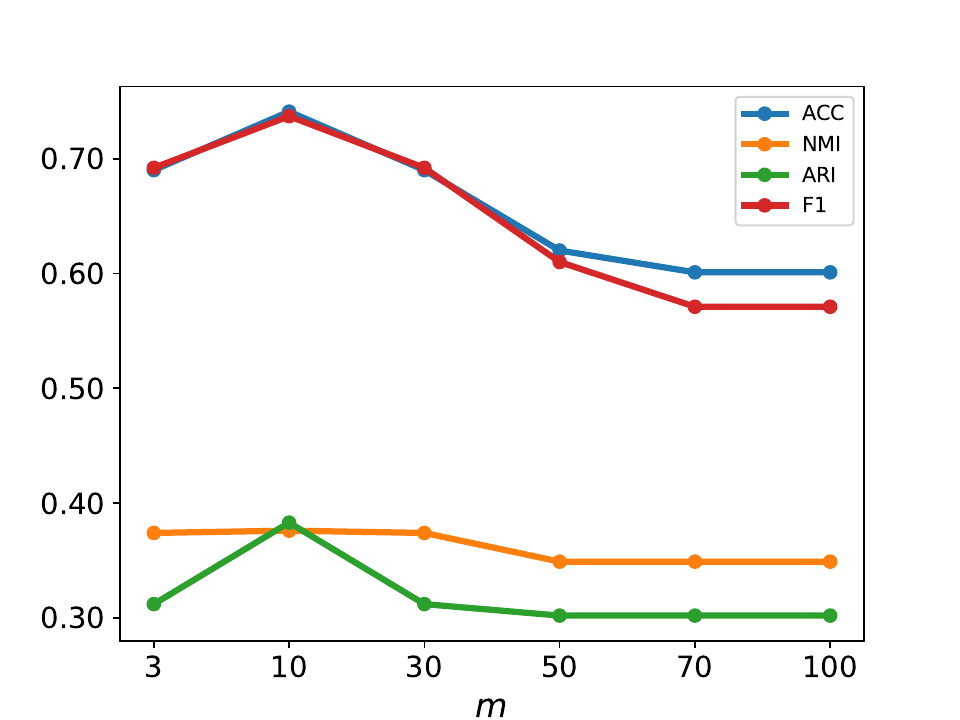}
		}
		\caption{Results on ACM and Pubmed with different anchor number $m$.}
		\label{PA2}
	\end{figure}

\subsubsection{Effect of Graph Filtering}
Graph filtering (marked as GF) is applied to integrate node attributes and topology in our method. In addition to theoretically demonstrating that the learned anchor graph from filtered representations is clustering-favorable, we also show this experimentally in Table \ref{ablation}. We can see that the performance of CDC \textit{w/o} GF drops about 10$\%$ on average, which validates the significance of graph filtering. We also observe  an increase in run time, which could be caused
by slower convergence due to the loss of cluster-ability.

\subsubsection{Robustness on Heterophily}
                   
\begin{table}[H]
  \centering
  \small
  \caption{Results on heterophilic graphs.}
  \setlength\tabcolsep{0.8pt}
    \begin{tabular}{ccccccccc}
    \toprule
          & \multicolumn{2}{c}{Texas} & \multicolumn{2}{c}{Cornell} & \multicolumn{2}{c}{Wisconsin} & \multicolumn{2}{c}{Squirrel} \\
    \midrule
    Methods & ACC   & NMI   & ACC   & NMI   & ACC   & NMI   & ACC   & NMI \\
    \midrule
    CDRS  & 0.599  & 0.154  & -     &   -    & \textcolor[rgb]{ .267,  .447,  .769}{\textbf{0.562 }} & 0.137  & -     & - \\
    CGC   & \textcolor[rgb]{ .267,  .447,  .769}{\textbf{0.615 }} & \textcolor[rgb]{ .267,  .447,  .769}{\textbf{0.215 }} & \textcolor[rgb]{ .267,  .447,  .769}{\textbf{0.446 }} & \textcolor[rgb]{ .267,  .447,  .769}{\textbf{0.141 }} & 0.559  & \textcolor[rgb]{ .267,  .447,  .769}{\textbf{0.230 }} & \textcolor[rgb]{ .267,  .447,  .769}{\textbf{0.272 }} & \textcolor[rgb]{ .267,  .447,  .769}{\textbf{0.030 }} \\
    CDC   & \textcolor[rgb]{ 1,  0,  0}{\textbf{0.672 }} & \textcolor[rgb]{ 1,  0,  0}{\textbf{0.293 }} & \textcolor[rgb]{ 1,  0,  0}{\textbf{0.514 }} & \textcolor[rgb]{ 1,  0,  0}{\textbf{0.142 }} & \textcolor[rgb]{ 1,  0,  0}{\textbf{0.637 }} & \textcolor[rgb]{ 1,  0,  0}{\textbf{0.318 }} & \textcolor[rgb]{ 1,  0,  0}{\textbf{0.279 }} & \textcolor[rgb]{ 1,  0,  0}{\textbf{0.043}}  \\
    \bottomrule
    \end{tabular}%
  \label{HTre}%
\end{table}%


In some real-world applications, graphs could be heterophilic, where connected nodes tend to have different labels \cite{CGC}. To show the robustness of CDC on heterophily, we report the results on several popular heterophilic graphs, including Texas, Cornell, Wisconsin \cite{Geom-GCN}, Squirrel \cite{squirel}. As shown in Table \ref{HTre}, our proposed CDC dominates recent SOTA methods, CDRS \cite{CDRS} and CGC \cite{CGC}. Although the used low-pass filter is considered less effective for heterophily, CDC still performs well because of the high-quality anchors and clustering-friendly graph. In fact, there are few graph clustering methods designed for heterophily. Further work, such as developing an omnipotent filter, could greatly contribute to improving clustering on heterophilic graphs.
\begin{figure}[H]
    \centering
    \subfigure[ACM]{
    \includegraphics[width=0.8\linewidth]{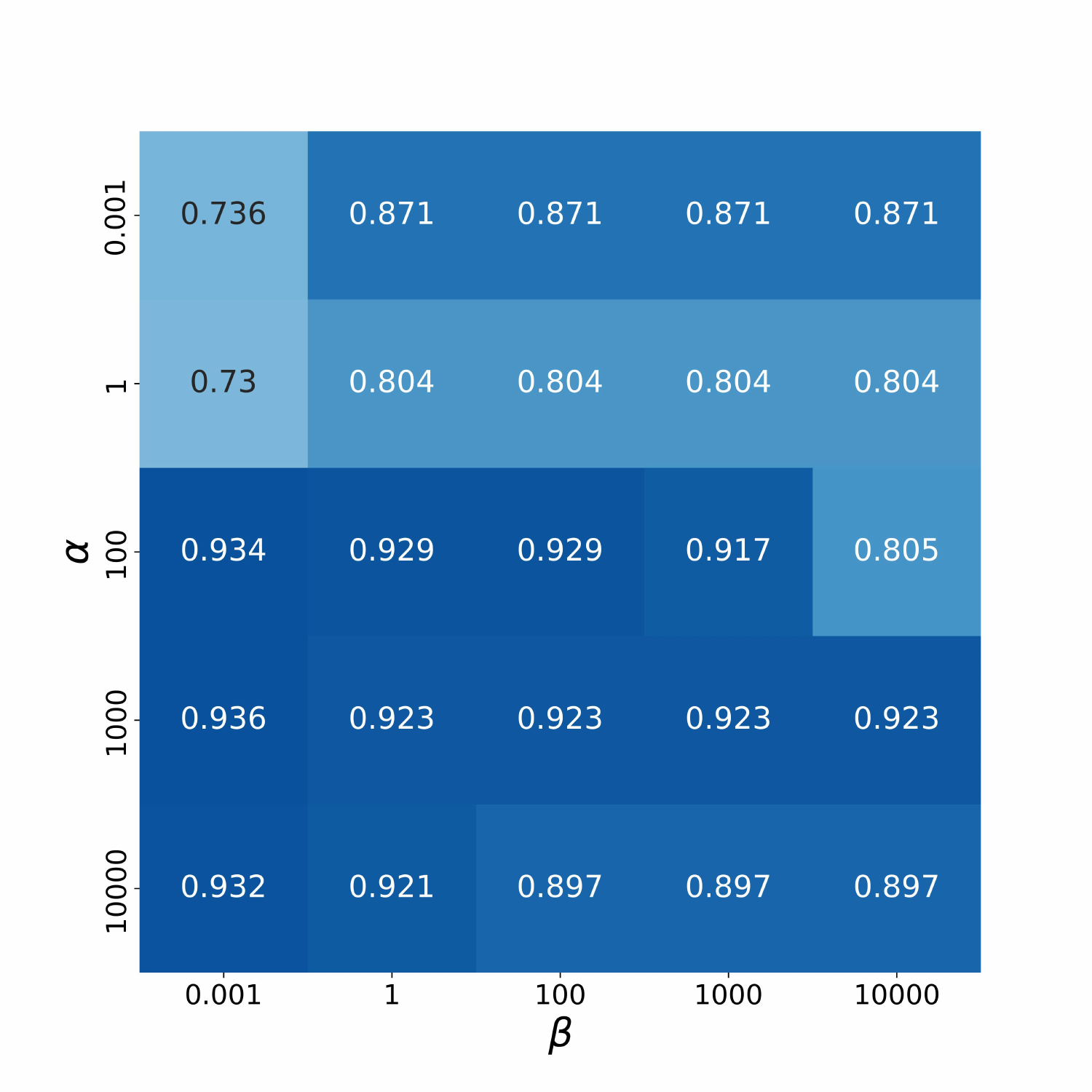}
    }
    \subfigure[Pubmed]{
    \includegraphics[width=0.8\linewidth]{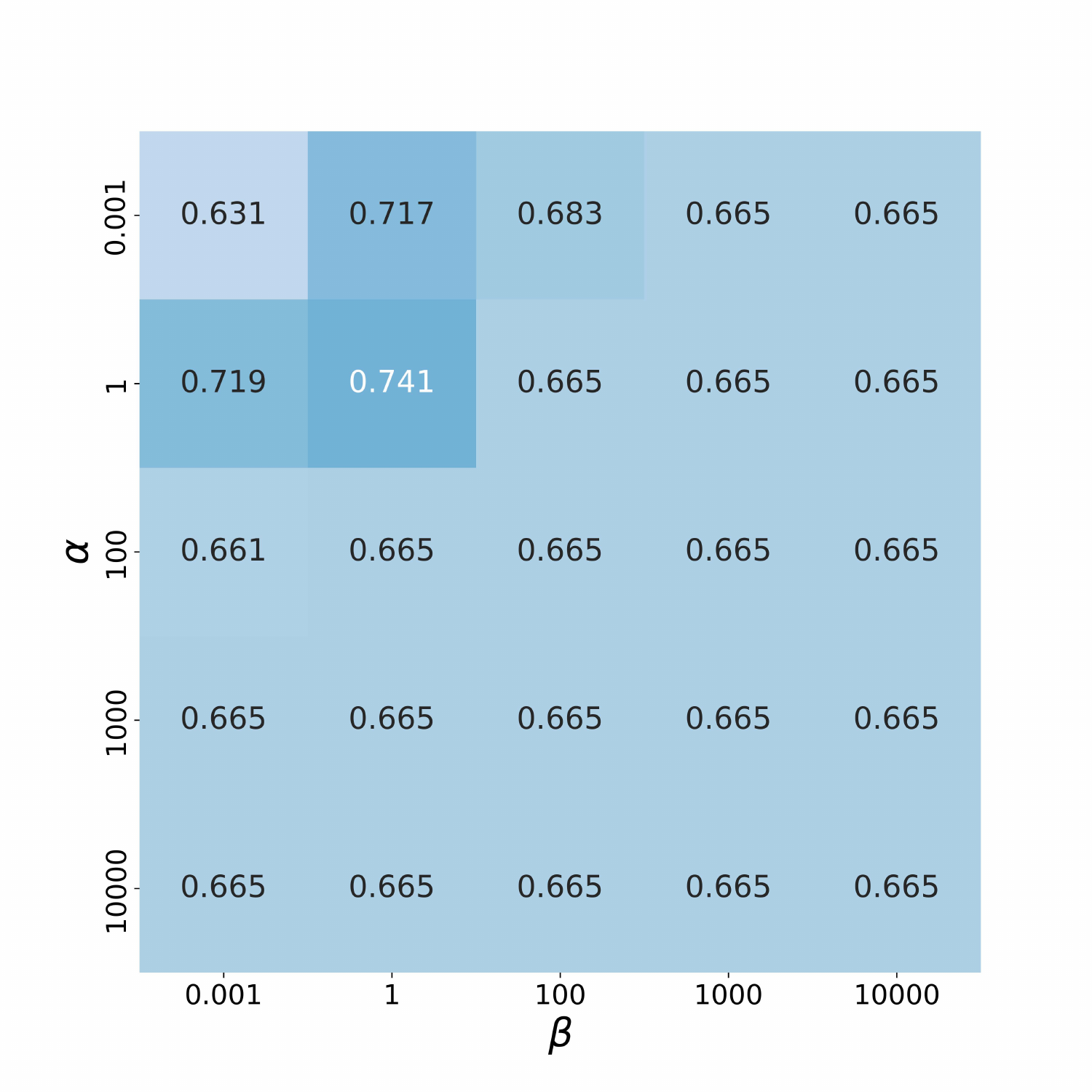}
    }
    \caption{Accuracy on ACM and Pubmed with different $\alpha$ and $\beta$.}
    \label{PA1}
\end{figure}

In addition, we  visualize the objective function value of CDC for ACM, Pubmed, and YTF-31 in Fig. \ref{lossv}. It can be seen that the losses rapidly converge.
  \begin{figure}[!htbp]
     
     \centering
     \subfigure[loss of CDC on ACM]{
     \includegraphics[width=0.85\linewidth]{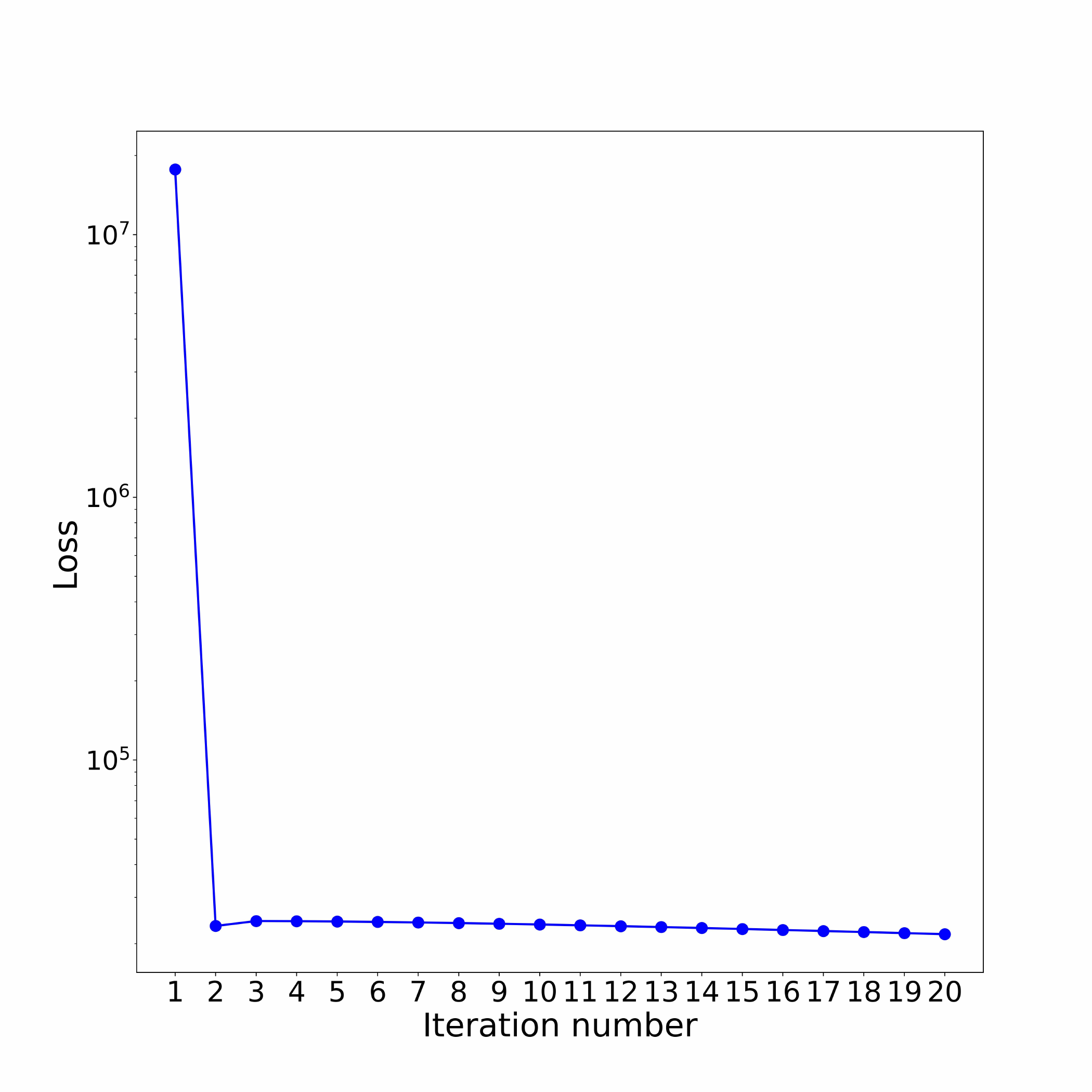}
     }
     \centering
     \subfigure[loss of CDC on Pubmed]{
     \includegraphics[width=0.85\linewidth]{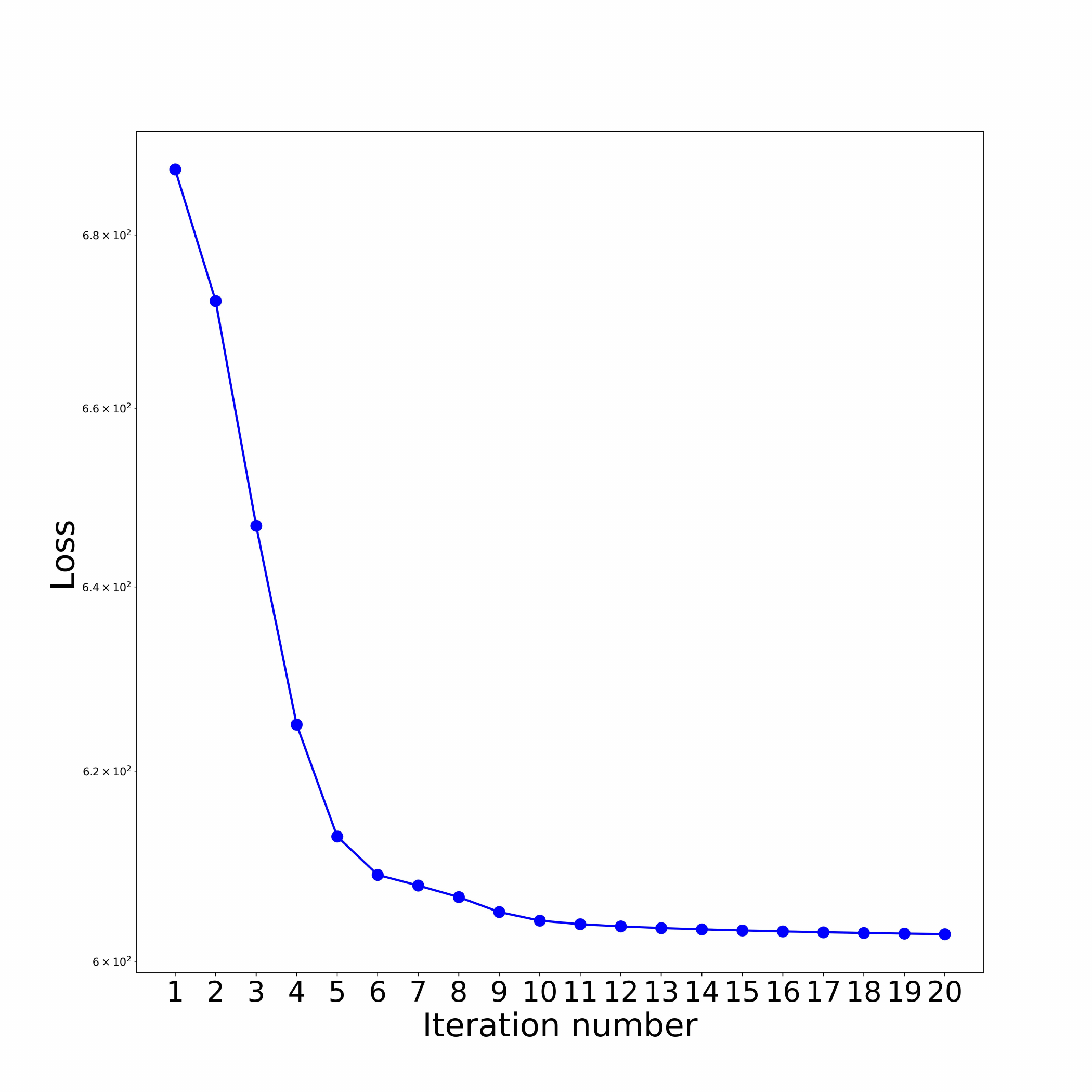}
     }
     \centering
     \subfigure[loss of CDC on YTF-31]{
     \includegraphics[width=0.85\linewidth]{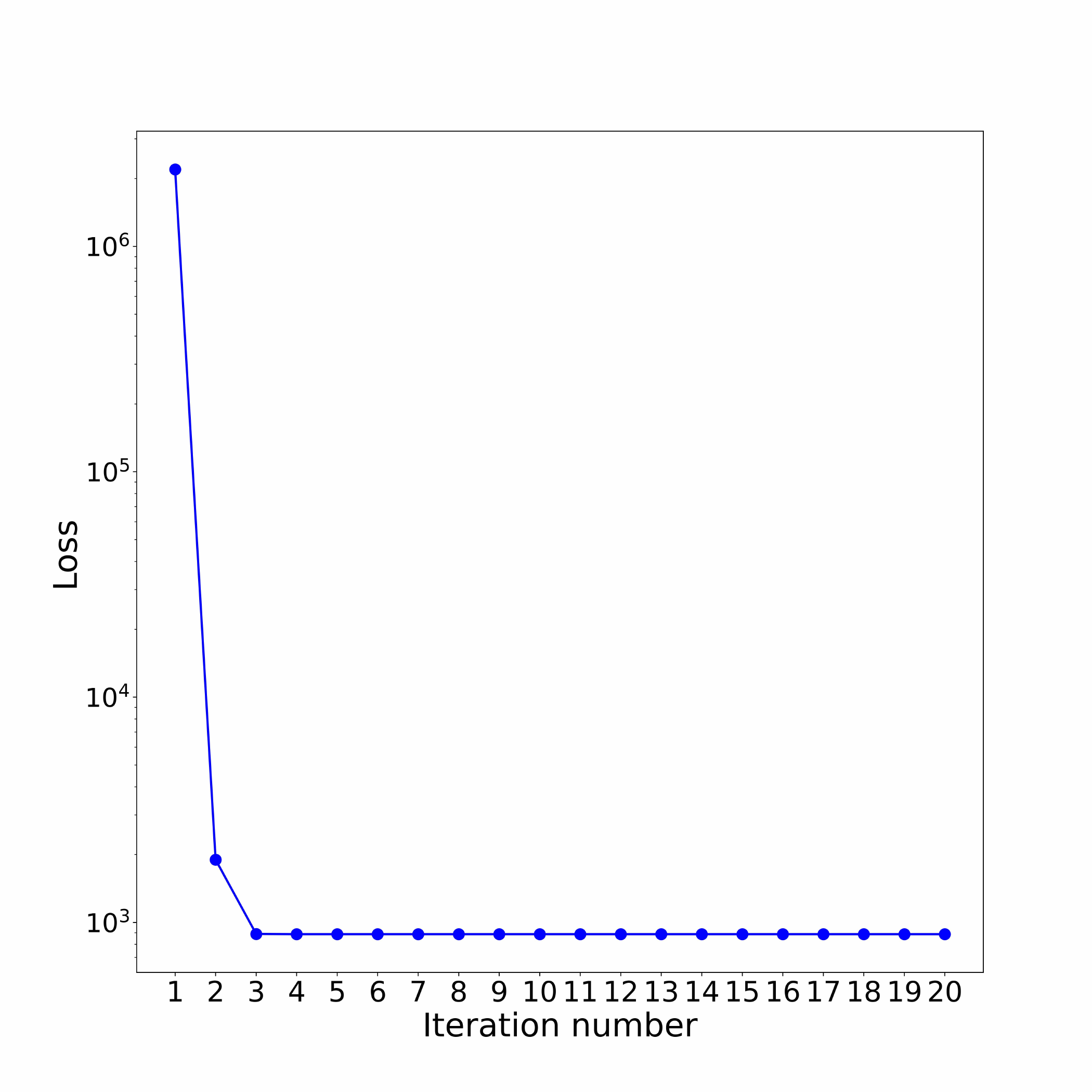}
     }
     \caption{The objective function value of CDC.}
     \label{lossv}
 \end{figure}


\section{Parameter Analysis} \label{PA}

There are two trade-off parameters, $\alpha$ and $\beta$, in our model. As shown in Figure \ref{PA1}, although CDC performs well across a wide range of $\alpha$ and $\beta$, fine-tuning further enhances its performance. $\beta$ has less impact than $\alpha$, indicating that the similarity-preserving regularizer is more important. The proposed CDC has linear complexity, so fine-tuning procedures require minimal time.


\section{Conclusion}
In this paper, we propose a simple framework for clustering complex data that is readily applicable to graph and non-graph, as well as 
 multi-view and single-view data. The developed method has linear complexity and strong theoretical properties. Through graph filtering, we integrate deep structural information and learn representations that enhance cluster-ability. In particular, a similarity-preserving regularizer is designed to adaptively generate high-quality anchors, which alleviates the burden and randomness of anchor selection. CDC demonstrates its effectiveness and efficiency with impressive results on 14 complex datasets. In particular, it even exceeds the performance of many complex GNN-based methods. In light of the simplicity of the proposed framework and its effectiveness on various types of data, this work could have a broad impact on the clustering community and a high potential for deployment in real applications. One potential limitation of CDC is that it might not handle high-dimensional data efficiently, as anchor generation has cubic complexity with respect to sample dimension.






\nocite{langley00}

\bibliography{example_paper}
\bibliographystyle{IEEEtran}



\begin{IEEEbiography}[{\includegraphics[width=1in,height=1.25in,clip,keepaspectratio]{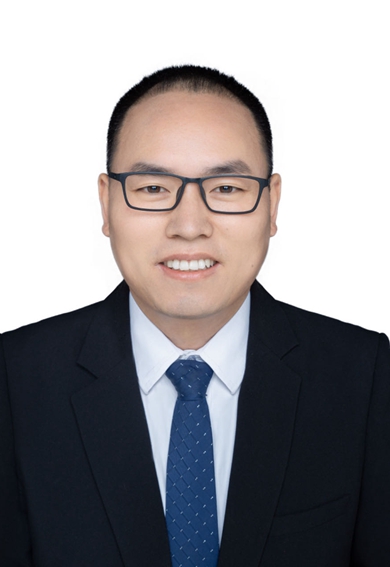}}]{Zhao Kang}
(Member, IEEE) received the Ph.D. degree in computer science from Southern Illinois University Carbondale, Carbondale, IL, USA, in 2017.
He is currently an Associate Professor at the School of Computer Science and Engineering, University of Electronic Science and Technology of
China, Chengdu, China. He has published over 100 research papers in top-tier conferences and journals, including ICML, NeurIPS, AAAI, IJCAI,
ICDE, CVPR, SIGKDD, ECCV, IEEE Transactions on Cybernetics, IEEE Transactions on Image Processing, IEEE Transactions on Knowledge and Data
Engineering, and IEEE Transactions on Neural Networks and Learning Systems. His research interests are machine learning and natural language processing. Dr. Kang has been an AC/SPC/PC Member or a Reviewer for a number of top conferences, such as NeurIPS, ICLR, AAAI, IJCAI, CVPR,
ICCV, ICML, and ECCV. He regularly serves as a reviewer for the Journal of Machine Learning Research, IEEE TRANSACTIONS ON PATTERN
ANALYSIS AND MACHINE INTELLIGENCE, IEEE TRANSACTIONS ON NEURAL NETWORKS AND LEARNING SYSTEMS, IEEE TRANSACTIONS
ON CYBERNETICS, IEEE TRANSACTIONS ON KNOWLEDGE AND DATA ENGINEERING, and IEEE TRANSACTIONS ON MULTIMEDIA.
\end{IEEEbiography}

\begin{IEEEbiography}[{\includegraphics[width=1in,height=1.25in,clip,keepaspectratio]{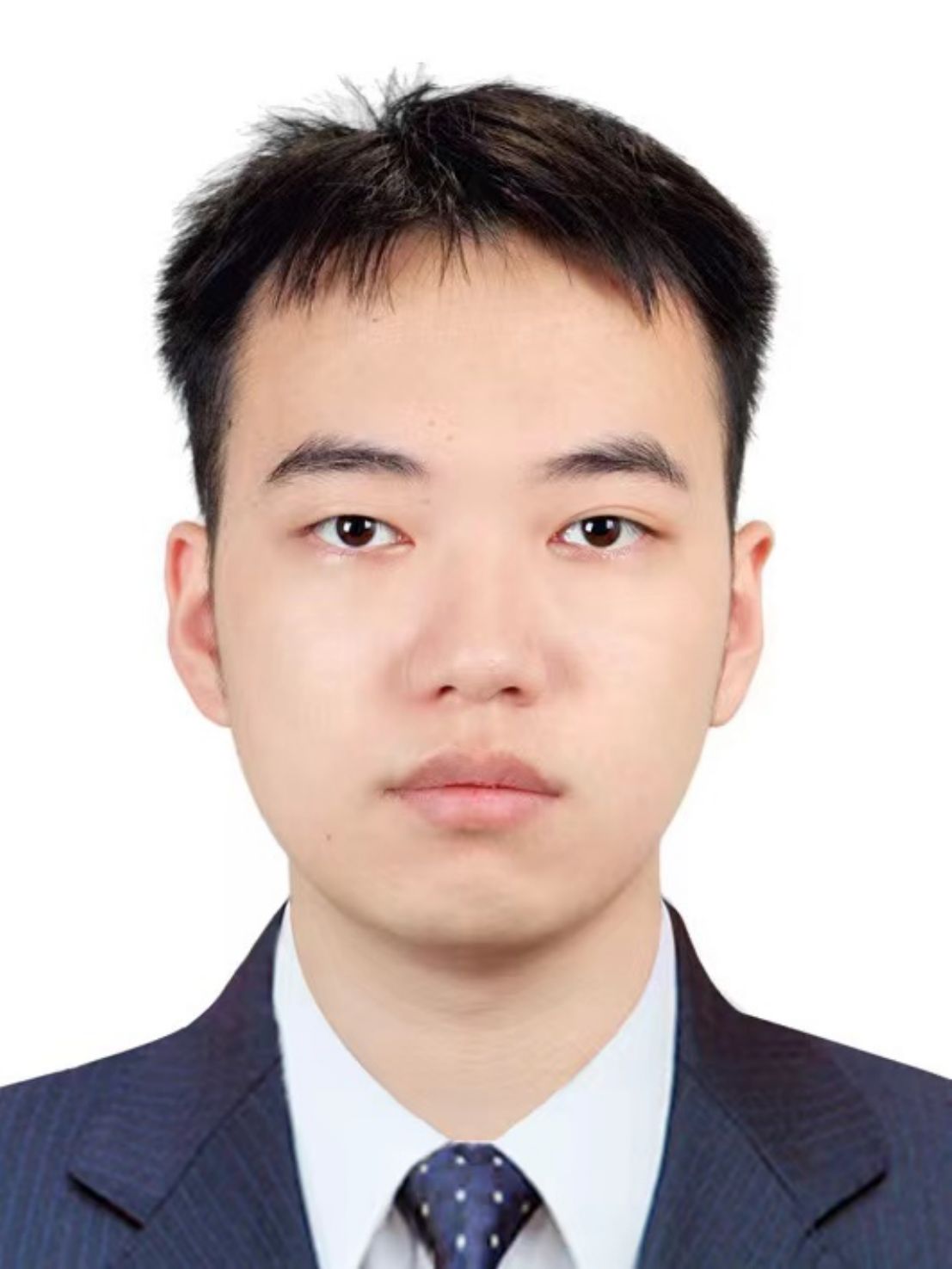}}]{Xuanting Xie}
received the B.Sc. degree in Software Engineering in
2021 and is currently studying for a Ph.D. degree with School of Computer
Science and Engineering, the University of Electronic Science and Technology
of China, Chengdu, China. His main research interests include graph learning and clustering.
\end{IEEEbiography}
\begin{IEEEbiography}[{\includegraphics[width=1in,height=1.25in,clip,keepaspectratio]{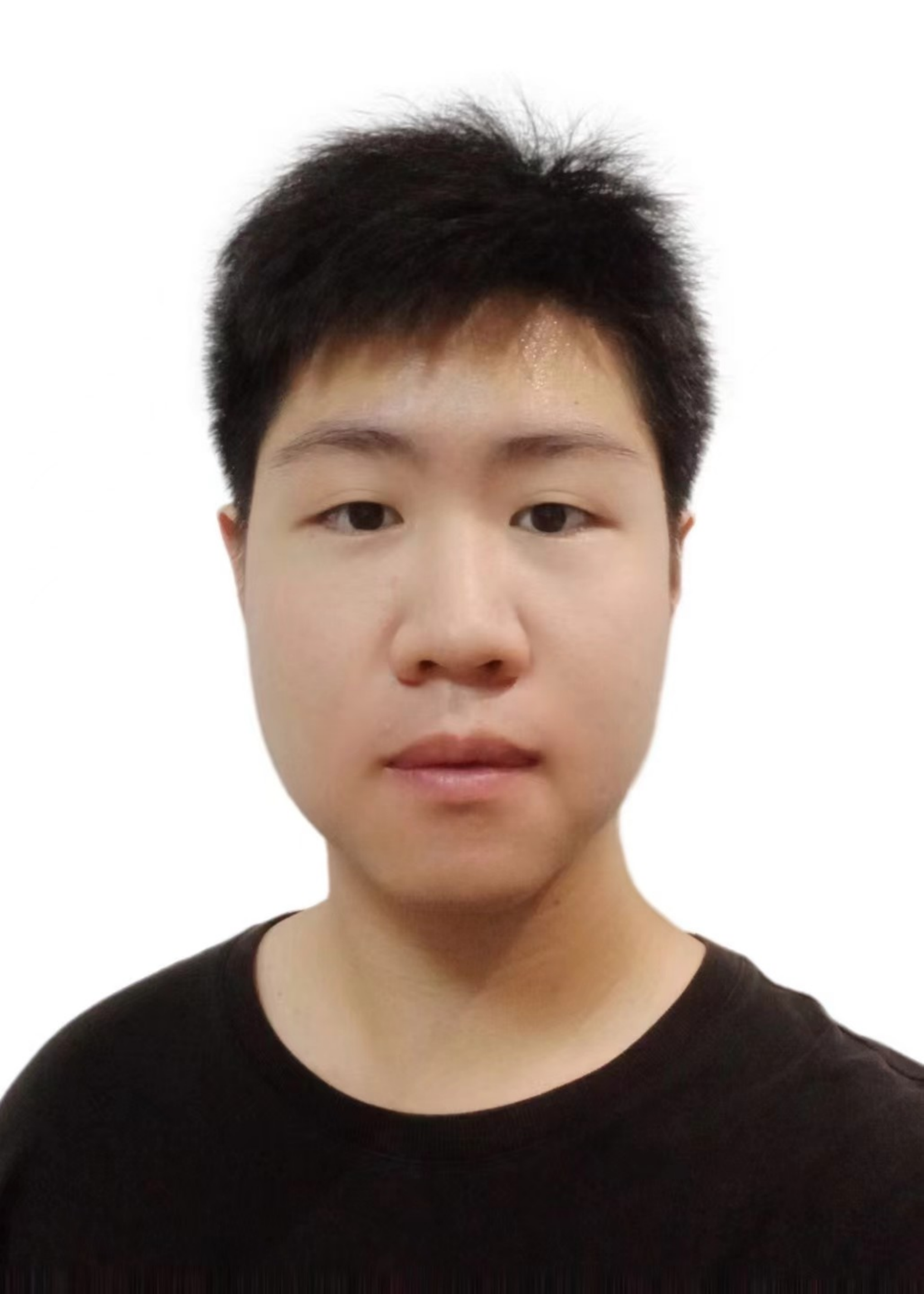}}]{Bingheng Li}
is currently an undergraduate student at the University of Electronic Science and Technology of China. His research focuses on graph clustering and graph data mining.
\end{IEEEbiography}
\begin{IEEEbiography}[{\includegraphics[width=1in,height=1.25in,clip,keepaspectratio]{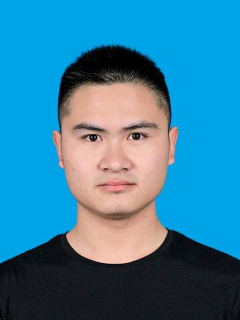}}]{Erlin Pan}
received the B.Sc. degree in computer science and technology in
2021 and is currently studying for a M.S degree with School of Computer
Science and Engineering, the University of Electronic Science and Technology
of China, Chengdu, China. His main research interests include graph learning,
multi-view learning, and clustering.
\end{IEEEbiography}

\vfill
\end{document}